\documentclass[10pt,twocolumn,letterpaper]{article}

\usepackage[pagenumbers]{cvpr} 

\DeclareMathOperator*{\argmin}{arg\,min}
\DeclareMathOperator*{\argmax}{arg\,max}

\usepackage{multirow}
\usepackage{tabularx}
\usepackage{bm}
\usepackage{amsthm}
\usepackage{amssymb}
\usepackage{algorithm}
\usepackage{algpseudocode}

\newcommand{\bred}[1]{{\color{red}\textbf{#1}}}
\newcommand{\bblue}[1]{{\color[HTML]{0070C0}\textbf{#1}}}
\newcommand{\white}[1]{{\color{orange}#1}}
\newcommand{\figvspace}{\vspace{0pt}}
\newcommand{\tabvspace}{\vspace{0pt}}

\newtheorem{proposition}{Proposition}[section]
\newtheorem{lemma}{lemma}[section]
\newtheorem{definition}{Definition}[section]

\definecolor{cvprblue}{rgb}{0.21,0.49,0.74}
\usepackage[pagebackref,breaklinks,colorlinks,allcolors=cvprblue]{hyperref}

\def\paperID{9439} 
\def\confName{CVPR}
\def\confYear{2026}

\title{Event Stream Filtering via Probability Flux Estimation}

\author{Jinze Chen, Wei Zhai, Yang Cao, Bin Li, Zheng-Jun Zha\\
University of Science and Technology of China\\
{\tt\small chjz@mail.ustc.edu.cn, \{wzhai056, forrest, binli, zhazj\}@ustc.edu.cn}
}

\begin{document}
\maketitle
\begin{abstract}

Event cameras asynchronously capture brightness changes with microsecond latency, offering exceptional temporal precision but suffering from severe noise and signal inconsistencies. Unlike conventional signals, events carry state information through polarities and process information through inter-event time intervals. However, existing event filters often ignore the latter, producing outputs that are sparser than the raw input and limiting the reconstruction of continuous irradiance dynamics. We propose the Event Density Flow Filter (EDFilter), a framework that models event generation as threshold-crossing probability fluxes arising from the stochastic diffusion of irradiance trajectories. EDFilter performs nonparametric, kernel-based estimation of probability flux and reconstructs the continuous event density flow using an O(1) recursive solver, enabling real-time processing. The Rotary Event Dataset (RED), featuring microsecond-resolution ground-truth irradiance flow under controlled illumination is also presented for event quality evaluation. Experiments demonstrate that EDFilter achieves high-fidelity, physically interpretable event denoising and motion reconstruction.

\end{abstract}    
\section{Introduction}
\label{sec:intro}

Event cameras are bio-inspired vision sensors that asynchronously report brightness changes with microsecond latency \cite{lichtsteiner2008128,taverni2018front,finateu20205}. Unlike conventional frame-based sensors that sample at uniform time intervals, the irregular, threshold-triggered event measurements avoid redundant sampling and thus fit into visual motion perception \cite{gallego2018unifying,vidal2018ultimate,falanga2020dynamic,gallego2020event}. But this new paradigm also magnifies the internal noise of the circuit, as pointed out in \cite{hu2021v2e,ding2023mlb}.


By analyzing the event generation model \cite{lichtsteiner2008128}, we observe that an event stream conveys two complementary types of information about the underlying brightness evolution, as illustrated in \cref{fig:principle}: (1) the \textbf{state information}, encoded by the discrete jump in logarithmic irradiance at the event time $I_{t_i}-I_{t_{i-1}}=p_{i}C_{}$ and, (2) the \textbf{process information}, described by the inequality $\sup_{t\in[t_{i-1},t_i)}|I_t-I_{t_{i-1}}|< C$, which constrains the latent irradiance trajectory between successive events. Here $I_t$ denotes the logarithm scene irradiance, $t_{i}\in\mathbb{R}$ and $p_{i}\in\{-1,1\}$ the timestamp and polarity of the $i$-th event and $C$ the contrast threshold. 
Classical signal processing approaches model continuous state variables effectively, but they are inherently unsuitable for discontinuous, asynchronous event polarities and cannot directly impose the above inequality constraint on the latent irradiance path. Consequently, most existing event filters ignore the process information and treat event polarities as if they were continuous signals, often producing outputs that are even sparser than the raw input and ultimately limiting the reconstruction of continuous irradiance dynamics.

\begin{figure}
    \centering
    \includegraphics[width=\linewidth]{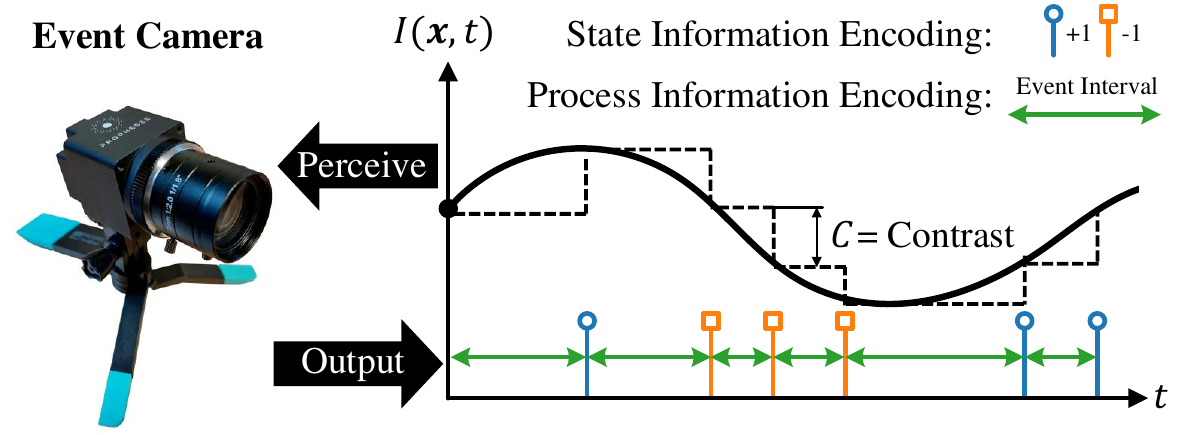}
    \caption{\textbf{Interpretation of the information encoded in an event stream.} Each event provides (1) state information through the discrete irradiance jump at the event time, and (2) process information by constraining the latent irradiance trajectory between events through the contrast-bound inequality.}
    \label{fig:principle}
    \figvspace
\end{figure}

\begin{figure*}
    \centering
    \begin{subfigure}{0.32\textwidth}
        \centering
        \includegraphics[width=.9\linewidth]{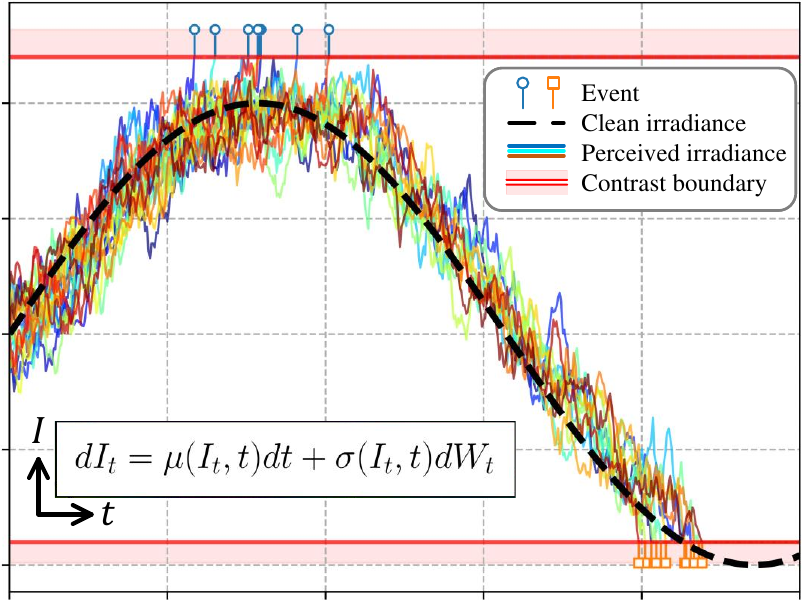}
        \caption{A stochastic specification of irradiance evolution, where clean irradiance and thermal noise determine the drift and diffusion of observed irradiance trajectories, and an event is triggered at the contrast boundary.}
    \end{subfigure}%
    \hfill
    \begin{subfigure}{0.34\textwidth}
        \centering
        \includegraphics[width=.9\linewidth]{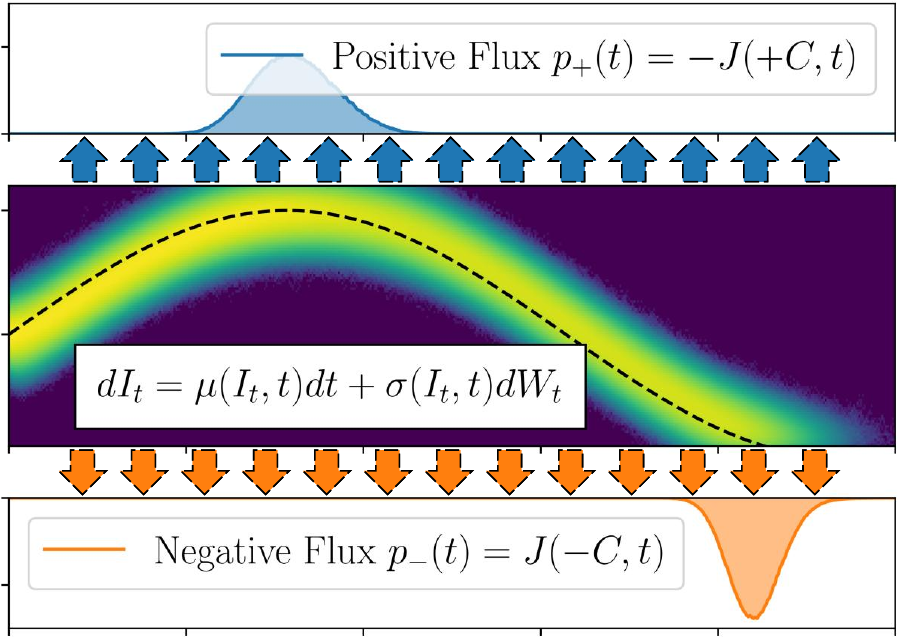}
        \caption{A distributional specification obtained by counting trajectory densities over irradiance–time space. The incoming probability flux at the contrast boundary corresponds to the instantaneous probability of generating an event.}
    \end{subfigure}
    \hfill
    \begin{subfigure}{0.32\textwidth}
        \centering
        \includegraphics[width=.9\linewidth]{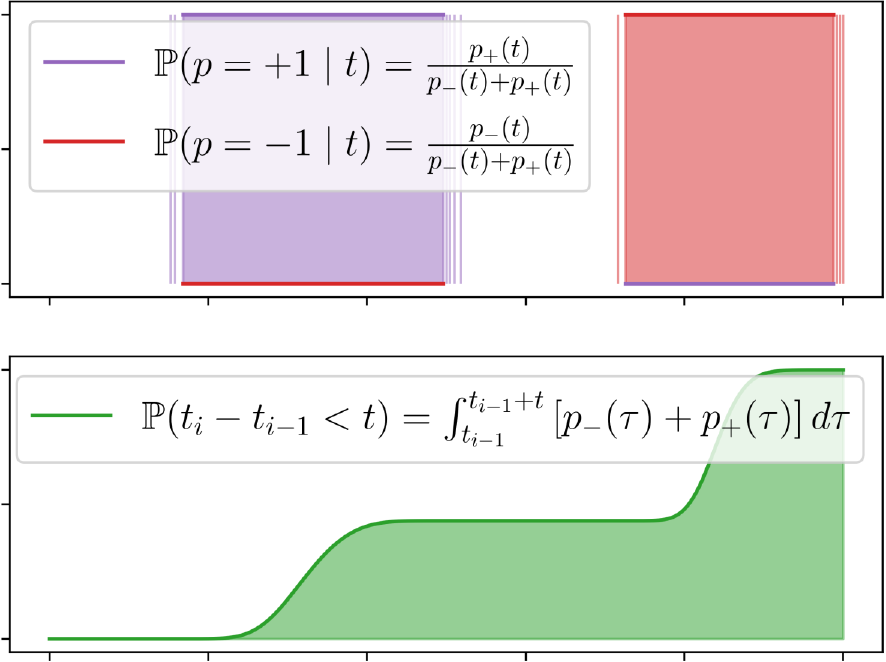}
        \caption{The distributions of discrete event polarities conditional on event time, and continuous inter-event times, obtained respectively from the ratio of polarity-specific boundary fluxes and their temporal integrals.}
    \end{subfigure}
    \caption{Stochastic analysis of event generation for a single pixel.}
    \label{fig:probflux}
    \figvspace
\end{figure*}

To bridge this gap, we revisit event generation from a stochastic process perspective. Intuitively, an event is triggered whenever the observed brightness trajectory moves far enough to cross the ON/OFF threshold. This trajectory is affected by thermal noise and exhibits randomness over time. The likelihood of this crossing is not determined by the past events themselves but how the brightness is flowing toward the threshold at any moment. This directional tendency of irradiance trajectories to cross the boundary is precisely what the probability flux (or current) represents.

Formally, although events appear as discontinuities, the irradiance evolves continuously and can be modeled as a threshold-crossing stochastic process. We find that both (1) the continuous distribution of \textbf{inter-event intervals} $\mathbb{P}(t_i-t_{i-1}<t)=\mathbb{P}(\sup_{\tau\in[t_{i-1},t_{i-1}+t)}|I_\tau-I_{t_{i-1}}|\geq C)$ and, (2) the discrete distribution of \textbf{event polarities} $\mathbb{P}(p_i=\pm1\mid t_i)=\mathbb{P}(I_{t_{i}}-I_{t_{i-1}}=\pm C~ \wedge~ \sup_{t\in[t_{i-1},t_i)}|I_t-I_{t_i}|<C)$ are determined by the threshold-crossing probability fluxes, as shown in \cref{fig:probflux}. Estimating this flux from discrete, noisy events provides a generative formulation of event filtering by resampling a clean, filtered event stream.

Building on this principle, we propose a generative event filtering framework termed Event Density Flow Filter (EDFilter) that estimates threshold-crossing probability fluxes in real time. The proposed EDFilter models the spatiotemporal density of events using a nonparametric kernel smoother optimized via recursive maximum likelihood estimation. This formulation leads to an $O(1)$ recursive solver, enabling online operation at sensor time scales.

We further propose the Rotary Event Dataset (RED), which provides microsecond-accurate ground-truth irradiance references under controlled motion and lighting conditions. Comprehensive experiments demonstrate that EDFilter achieves superior denoising accuracy and temporal fidelity compared with state-of-the-art methods, while maintaining theoretical interpretability and real-time efficiency.

In summary, the main contributions of this work are:

\begin{enumerate}
    \item The interpretation of event sensing as a threshold-crossing stochastic process is proposed, revealing that both inter-event intervals and event polarities are determined by probability fluxes at the contrast boundaries.
    \item A generative, probability-flux-based event filter EDFilter is proposed, which estimates boundary-crossing fluxes from discrete events using nonparametric, $O(1)$ recursive kernel smoothers, enabling real-time signal filtering.
    \item The RED dataset with precise microsecond ground truth irradiance references is proposed, which enables precise evaluation of temporal fidelity and supports research on physically accurate event sensing models.
    \item Extensive experiments on RED and existing datasets show that EDFilter achieves superior denoising accuracy and temporal precision.
\end{enumerate}

\section{Related Work}
\hspace*{\parindent}\textbf{Discrimitative Filters.} Most filters are discriminative. Density-based methods \cite{delbruck2008frame,liu2015design,khodamoradi2018n,feng2020event,zhang2023neuromorphic,guo2025ebf} pass only events with sufficient spatiotemporal neighbors. Motion-based methods \cite{mueggler2015lifetime,wang2019ev} fit local planes to estimate velocity and filter out events that are not motion-consistent. More recently, learning-based methods use MLP \cite{guo2022low}, CNN \cite{baldwin2020event}, or Transformer \cite{jiang2024edformer} on event patches to classify each event as signal or noise. The limitation of these methods is their discriminative nature: they can only remove events, not correct them or generate a new, clean stream.

\textbf{Generative Filters.} A few methods \cite{wang2020joint,duan2021guided,duan2021eventzoom,duan2023neurozoom} are generative, most notably those \cite{duan2021eventzoom,duan2023neurozoom} that convert events into stacked event frames and use a 3D-UNET to denoise and super-resolve these frames. While powerful, these methods sacrifice the core advantage of event cameras—temporal precision---by binning events into synchronous frames.

\textbf{Our Approach.} The work most closely related to ours is \cite{lin2022dvs}. Although designed for event simulation, it models event generation through stochastic diffusion of irradiance for continuous-time event simulation. Our method takes the opposite direction: we infer the diffusion from the observed events. Crucially, we show that events can be interpreted as direct samples of boundary probability fluxes, allowing us to bypass solving the full diffusion equation and derive a principled and efficient formulation for event modeling.

\section{Approach}

\subsection{Theoretical Foundation}\label{sec:theoflux}
This section derives the probability flux as a physical quantity for characterizing event distributions, as illustrated in \cref{fig:probflux}. This is achieved by modeling the evolving irradiance trajectory using a stochastic differential equation (SDE) with absorbing boundaries. For clarity, we focus on a single pixel and assume the process starts at $0$, but the formulation naturally extends to multi-pixel settings with spatial coupling and interference effects. Detailed derivations are provided in the supplementary material.

For a given event pixel, let $I_t$ denote the logarithmic irradiance intensity. Under realistic conditions---assumed to be Markovian and continuous---it follows the SDE:
\begin{equation}\label{eq:irradiancesde}
    dI_t=\mu(I_t,t)dt+\sigma(I_t,t)dW_t,
\end{equation}
where the drift $\mu(I_t,t)$ captures clean scene dynamics and the diffusion $\sigma(I_t,t)$ models thermal noise, with $W_t$ a standard Wiener process, as shown in \cref{fig:probflux}(a). Each pixel contains ON and OFF comparators that continuously compare $I_t$ against contrast thresholds $\pm C$. An event is triggered when $I_t$ exits the interval $[-C, C]$, with \textbf{timestamp} $t_s=\inf\{\tau>0\mid I_\tau\notin[-C,C]\}.$
The exit boundary determines the \textbf{polarity} $p\in\{+1,-1\}$.

Associated with this SDE is the Fokker-Planck equation \cite{elgin1984fokker}, which states the probability density function $\rho(I,t)$ of $I_t$ is the solution to the boundary value problem:
\begin{gather}\label{eq:bvpfokkerplanck}
    \partial_t\rho(I,t)=\frac{1}{2}\partial_I^2[\sigma^2(I,t)\rho(I,t)]-\partial_I[\mu(I,t)\rho(I,t)],\\\label{eq:bvpboundary}
    \text{subject to }\rho(-C,t)=\rho(C,t)=0, \quad \rho(I,0)=\rho(I),
\end{gather}
where $\rho(I)$ is the initial distribution. To simplify the expression, we introduce the probability flux density as
\begin{equation}\label{eq:probabilifyflux}
    J(I,t)=\mu(I,t)-\frac{1}{2}\partial_I[\sigma^2(I,t)\rho(I,t)],
\end{equation}
representing the instantaneous flow of probability mass, since it follows the continuity equation
\begin{equation}
    \partial_t\rho(I,t)+\nabla\cdot J(I,t)=0,
\end{equation}
and the outgoing fluxes at the boundaries,
\begin{equation}\label{eq:boundprobflux}
    p_{\pm}(t)=\mp J(\pm C,t),
\end{equation}
constitute the instantaneous ON/OFF event densities as shown in \cref{fig:probflux}(b). Hence, events can be interpreted as samples drawn from the leaking probability flux at the boundary. Since the number of leaked events should be complementary to the remaining trajectories, we have
\begin{equation}\label{eq:processlikelihood}
    \begin{split}
        \mathbb{P}(t_s\leq t)&=1-\int_{-C}^C\rho(I,t)dI,\\
        &=\int_0^t\left[p_+(\tau)+p_-(\tau)\right]d\tau.
    \end{split}
\end{equation}
The conditional distribution of polarities is derived as:
\begin{equation}\label{eq:statelikelihood}
    \mathbb{P}(p=\pm1\mid t)=\frac{p_\pm(t)}{\mathbb{P}(t_s=t)}=\frac{p_\pm(t)}{p_-(t)+p_+(t)},
\end{equation}
completing the proof of the results in \cref{fig:probflux}(c).

However, the solution of \cref{eq:bvpfokkerplanck,eq:bvpboundary} does not generally admit a closed-form expression \cite{elgin1984fokker}. To address this, we reformulate them with the failure rate \cite{finkelstein2008failure}, here as the \textbf{event density flow (EDF)}:
\begin{equation}\label{eq:eventdensityflow}
    \lambda_{\pm}(t)=\frac{p_{\pm}(t)}{1-\int_0^t[p_{+}(\tau)+p_{-}(\tau)]d\tau},
\end{equation}
which describes the instantaneous expected event rate.
The formulation of $\lambda_\pm(t)$ allows direct estimation from discrete events, since they are unconstrained nonnegative real-valued functions. In this work, we employ nonparametric temporal and spatial kernel-based smoothers to estimate these functions, as the following section shows.

\subsection{The EDFilter Framework}\label{sec:algorithm}
\begin{figure*}
  \centering
  \includegraphics[width=.95\linewidth]{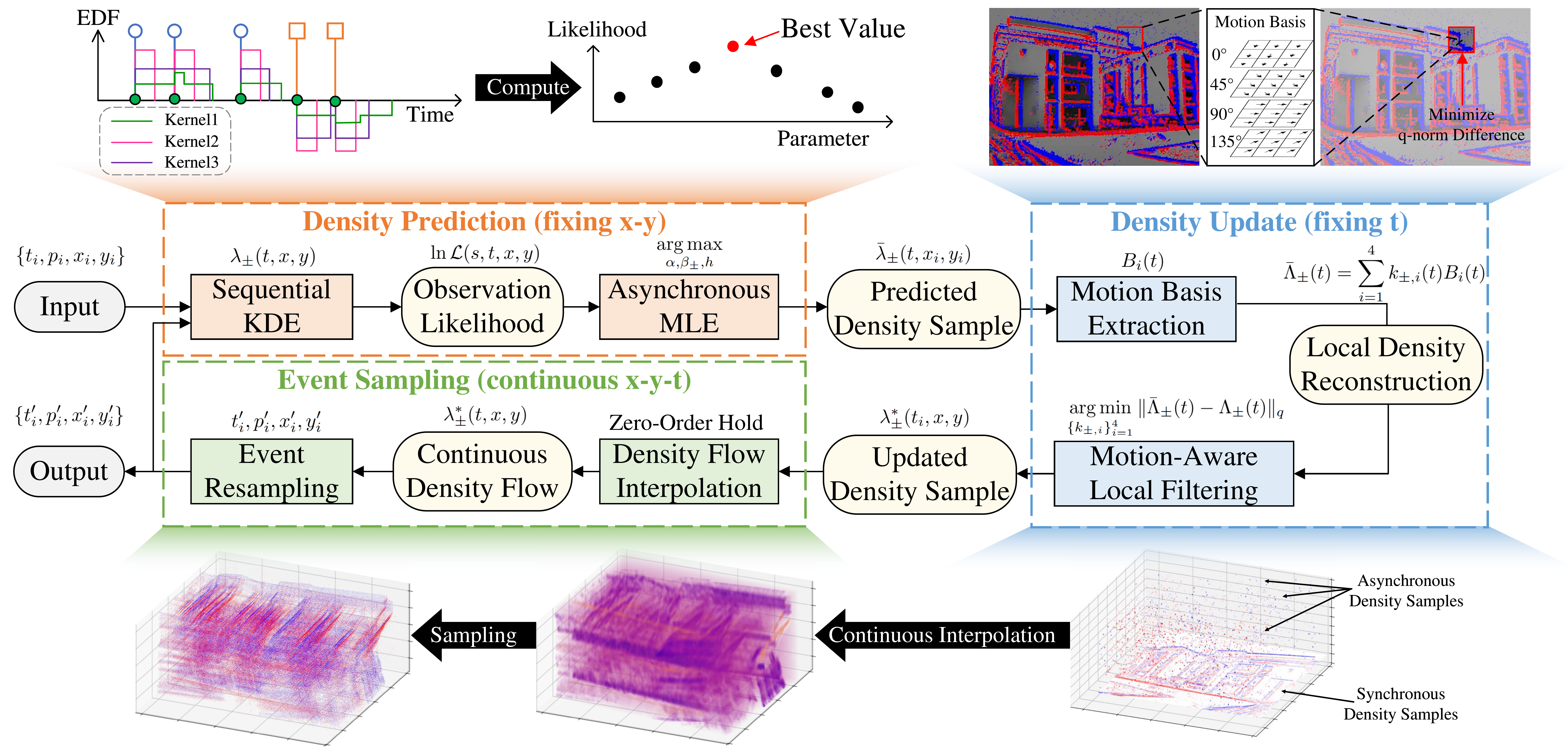}
   \caption{\textbf{The proposed EDFilter.} The density-prediction module sequentially applies KDE to the input events and selects the optimal kernel by maximizing the event-observation likelihood, producing a predicted density sample. The density-update module then spatially fuses these density samples using a motion-aware, sparsity-preserving local filter to obtain refined density estimates. Finally, the event-sampling module interpolates the continuous event-density flow via zero-order hold and resamples the filtered events for output. This output also triggers the density-prediction module in an application-dependent manner, mitigating the impact of abnormal integration.}
   \tiny
   \label{fig:method}
   \figvspace
\end{figure*}

The proposed EDFilter framework, illustrated in \cref{fig:method}, consists of three sequential components that respectively predict, update, and reconstruct the event density flow.

\subsubsection{Sequential Kernel Density Estimation\label{sec:temporal}}
The first step is to generate a continuous density prediction at any given time. This is achieved using a temporal kernel density estimator (KDE) for each pixel independently.
A rectangular kernel is adopted to model the contribution of sequential events, while both its height and bandwidth are treated as parameters to be estimated online. To account for noise, a constant false event rate is also introduced and optimized jointly. These defines our \textbf{temporal EDF model}:
\begin{gather}
    (\lambda_+(t),\lambda_-(t))=\begin{cases}
        (\psi(t),\beta_-), & \text{if } \psi(t) > 0\\
        (\beta_+,-\psi(t)), & \text{if } \psi(t) < 0\\
        (\beta_+,\beta_-), & \text{if } \psi(t) =0
    \end{cases}\label{eq:temporaledf}\\
    \psi(t)=\sum_{t_i<t}p_i\phi(t-t_i),\phi(t)=\begin{cases}
        \alpha, & \text{if } 0\leq t<h\\
        0, & \text{otherwise}
    \end{cases}
    \label{eq:temporalkde}
\end{gather}
where $\alpha$ and $h$ are the height and bandwidth of the rectangular kernel $\phi(t)$, $\psi(t)$ is the predicted density, and $\beta_\pm\geq0$ denote false event rates.
Unlike existing event filters, abnormal events are not hard-thresholded but still contribute to the continuous scene modeling process.

To determine optimal EDF parameters, we adopt the likelihood maximization strategy. If $N_e$ events are observed in the interval $(s,t]$ with times and polarities $\{(t_i,p_i)\}_{i=1}^{N_e}$, the log-likelihood is can be derived from \cref{eq:processlikelihood,eq:statelikelihood} as
\begin{equation}\label{eq:likelihood}
    \begin{split}
        \ln{\mathcal{L}(s,t)}=\sum_{i=1}^{N_e}\ln{\lambda_{p_i}(t_i)}
        - \int_{s}^{t}[\lambda_+(t)+\lambda_-(t)]dt,
    \end{split}
\end{equation}
which combines the contributions of state and process likelihoods, where $\lambda_p(t)=\lambda_\pm(t)$ for $p=\pm1$. To derive closed-form solutions, we also impose a prior distribution:
\begin{equation}
    f(\alpha,\beta_\pm,h)\propto\frac{\gamma}{h^2}\exp{\left(-\frac{\gamma}{h}\right)},
\end{equation}
where $\gamma$ is interpreted as the expected event interval. Maximizing the posterior likelihood  yields the predicted EDF:
\begin{equation}\label{eq:likelihoodmax}
    \bar{\lambda}_\pm(t)=\lambda_{\pm}(t)|_{\argmax_{(\alpha,\beta_\pm,h)} [\ln{\mathcal{L}(s,t)+\ln{f(\alpha,\beta_\pm,h)}}]}.
\end{equation}



\subsubsection{Motion-Aware Local Kernel Smoothing\label{sec:spatial}}
\begin{figure*}[!ht]
    \centering
    \begin{subfigure}{.50\textwidth}
        \centering
        \includegraphics[width=\linewidth]{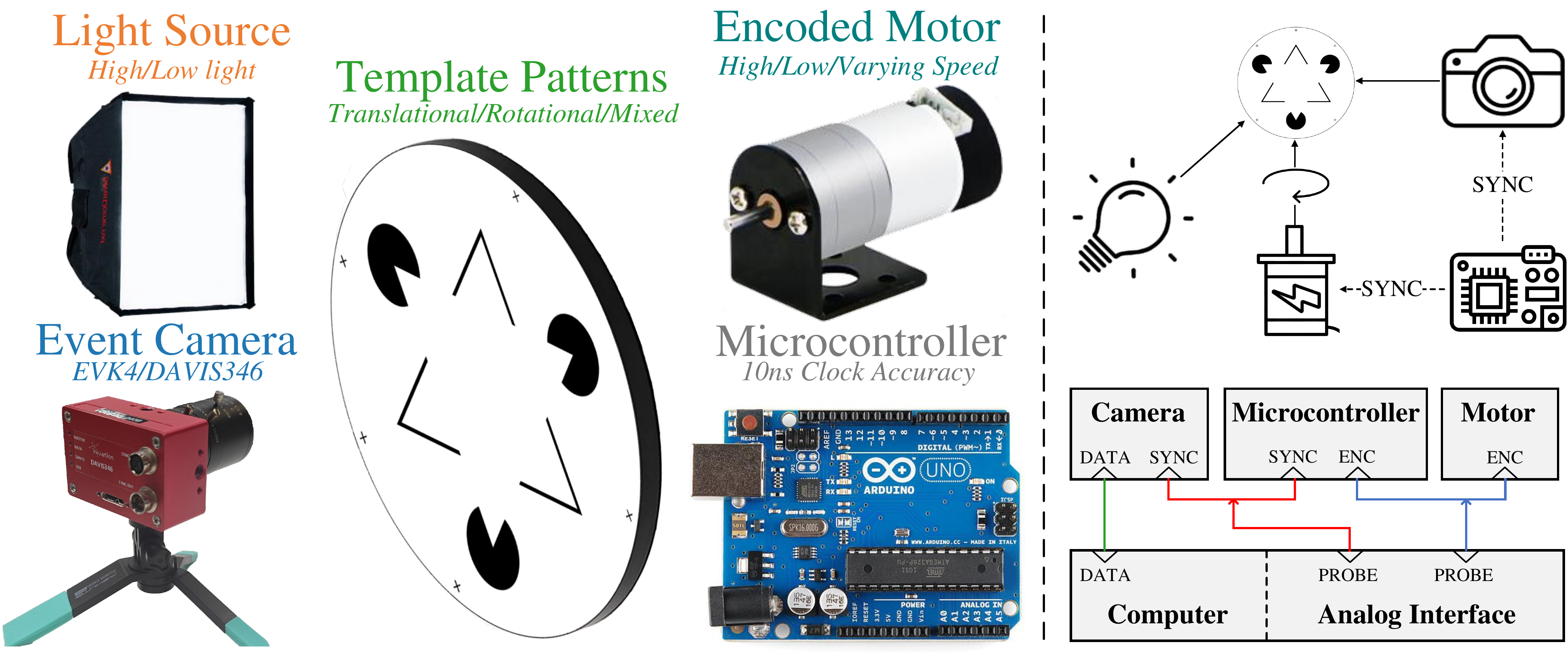}
        \caption{\textbf{RED collection details.} A variety of synchronized intensity frames and events are generated by varying each component. Synchronization is achieved by attaching analog interface to probe the synchronization and encoder signal.
        }
    \end{subfigure}
    \hfill
    \begin{subfigure}{.47\textwidth}
        \centering
        \includegraphics[width=\linewidth]{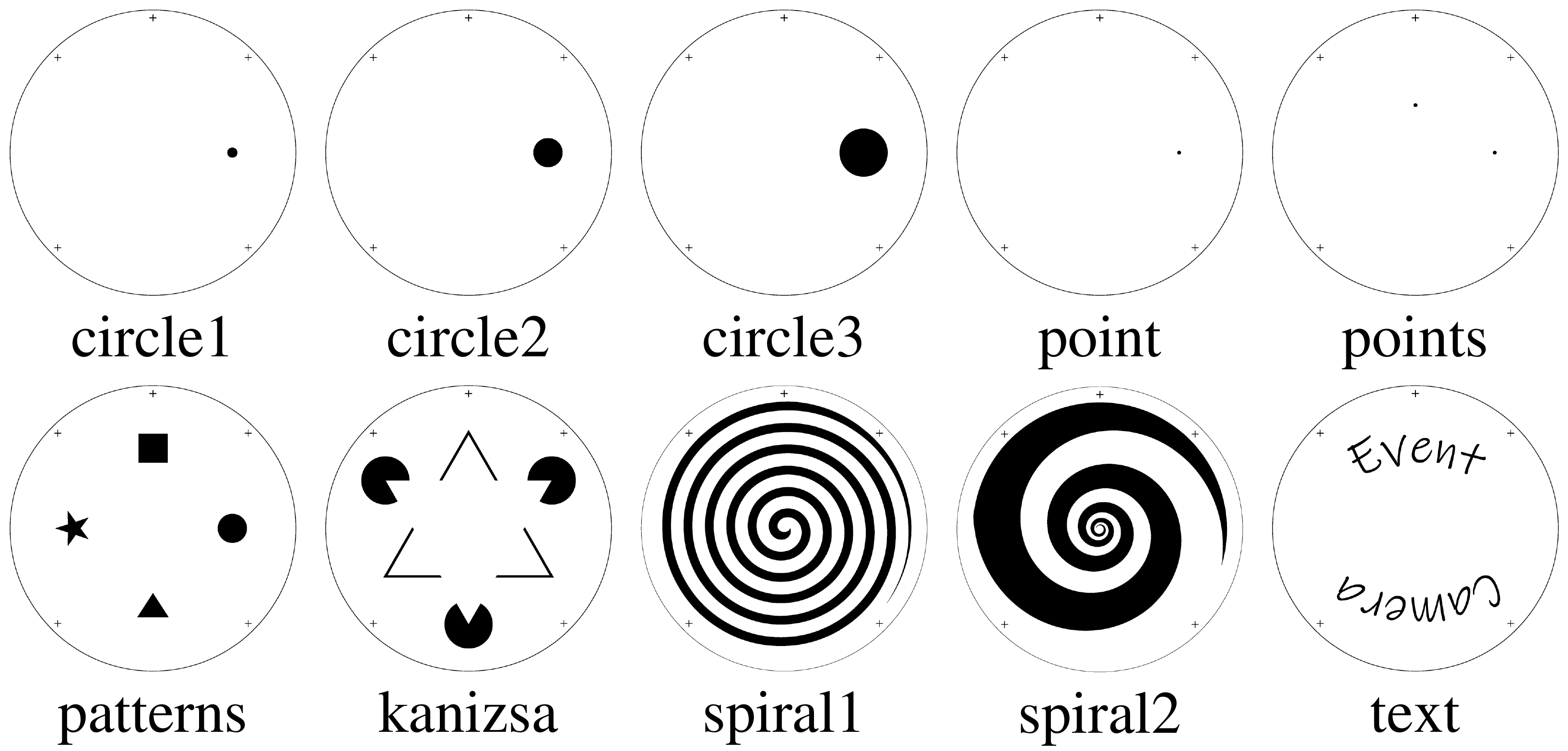}
        \caption{\textbf{Printed binary patterns.} In addition to the rotational motion types, the \textit{spiral1} and \textit{spiral2} patterns additionally simulate linear motion types. The \textit{point} and \textit{points} patterns are only used for evaluating point tracking. }
    \end{subfigure}
    \caption{System setup of the Rotary Event Dataset benchmark.}
    \label{fig:systemred}
    \vspace{-5pt}
\end{figure*}

The temporally predicted EDF $\bar{\lambda}_\pm(t)$ may exhibit structural inconsistencies due to noise or threshold mismatch. To enforce local consistency, we introduce a \textbf{spatial EDF model} that performs motion-aware smoothing by incorporating neighborhood information.
Unlike standard isotropic filters, spatial EDF patterns are directional and sparse, typically corresponding to moving edge intensities. Our model exploits this structure by imposing a directional motion prior on a $3\times3$ pixel neighborhood and expressing the smoothed EDF as a linear combination of directional basis patterns, whose coefficients are optimized via reconstruction error minimization.

We select 4 directional motion bases as
\begin{equation}\label{eq:spatialbasis}
\begin{split}
    &B_1(t) = \begin{bmatrix}
        1 & 0 & 0 \\ 
        0 & b_1(t) & 0 \\
        0 & 0 & 1
    \end{bmatrix}
    B_2(t) = \begin{bmatrix}
        0 & 1 & 0 \\
        0 & b_2(t) & 0 \\
        0 & 1 & 0
    \end{bmatrix}\\
    &B_3(t) = \begin{bmatrix}
        0 & 0 & 1 \\
        0 & b_3(t) & 0 \\
        1 & 0 & 0
    \end{bmatrix}
    B_4(t) = \begin{bmatrix}
        0 & 0 & 0 \\
        1 & b_4(t) & 1 \\
        0 & 0 & 0
    \end{bmatrix},
\end{split}
\end{equation}
representing 0°, 45°, 90°, and 135° motion patterns, with $0\leq b_i(t)\leq1$ denoting the contribution to the center pixel. Then the smoothed EDF around pixel $(x,y)$ is expressed
\begin{equation}\label{eq:spatialcontrib}
    \Lambda_\pm(t)=\sum_{i=1}^4k_{\pm,i}(t)B_i(t),
\end{equation}
where the optimal coefficients $\{k^*_{\pm,i}(t)\}_{i=1}^4$ are obtained by minimizing the $q$-norm reconstruction error:
\begin{equation}\label{eq:spatialopt}
    \{k^*_{\pm,i}(t)\}_{i=1}^4=\argmin_{\{k_{\pm,i}(t)\}_{i=1}^4}\|\Lambda_\pm(t)-\bar{\Lambda}_\pm(t)\|_q,
\end{equation}
where $\bar{\Lambda}_\pm(t)=\left(\bar{\lambda}_\pm(t,x+i,y+j)\right)_{-1\leq i,j\leq 1}$ and $1\leq q\leq 2$. We adaptively set $b_i(t)=\frac{1}{K(t)}$ if either neighboring side of the predicted EDF in the $i$-th direction is nonzero, otherwise $b_i(t)=0$, with $K(t)$ being the number of effective bases. The updated central density flow is given by:
\begin{equation}\label{eq:spatialout}
    \lambda^*_\pm(t,x,y)=\sum_{i=1}^4k_{\pm,i}(t)b_i(t).
\end{equation}


\subsubsection{Density Flow Reconstruction and Event Sampling}
Being able to evaluate $\lambda_\pm^*(t,x,y)$, the next thing to do is to reconstruct the continuous EDF.
We employ a hybrid sampling scheme combining event-triggered asynchronous local sampling and periodic global sampling, preserving transient dynamics without abnormal integral drift.

An event triggers an update of the local $3\times3$ EDF neighborhood $\left(\lambda^*_\pm(t,x+i,y+j)\right)_{-1\le i,j\le1}$, which depends on predicted EDF values from a $5\times5$ region $\left(\bar{\lambda}_\pm(t,x+i,y+j)\right)_{-2\le i,j\le2}$.
Additionally, a global update for all pixels is performed every $T_s$ seconds, regardless of event activity.
A zero-order interpolator is then used to obtain a continuous EDF representation, enabling event resampling by treating the process as a piecewise-homogeneous Poisson process.

\subsubsection{An \texorpdfstring{$O(1)$}{O(1)} Recursive Solver}
All components of EDFilter are designed for asynchronous real-time operation with constant-time complexity $O(1)$.

\textbf{Temporal EDF Prediction.} From \cref{eq:temporalkde}, if events are viewed as a sequence of polarity-modulated spike train 
\begin{equation}
    \mathcal{E}(t)=\sum_i p_i\delta(t-t_i),
\end{equation}
we have $\psi(t)=\mathcal{E}(t)*\phi(t)$ with $*$ being the convolution. As this is a linear time-invariant (LTI) system, it admits the state-space representation
\begin{equation}
    \frac{d\psi(t)}{dt}=\alpha(\mathcal{E}(t)-\mathcal{E}(t-h)).
\end{equation}
Fixing $h$, $\lambda_\pm(t)$ and its integral follow from \cref{eq:temporaledf}.
The parameters $\alpha$ and $\beta_\pm$ are optimized independently, while $h$ is chosen via a lookup table with discrete candidates maximizing likelihood, ensuring an $O(1)$ implementation.

\textbf{Spatial EDF Update.} 
Given the proven uniqueness of \cref{eq:spatialout}, solving \cref{eq:spatialopt} is a convex optimization problem with fixed-time convergence \cite{boyd2004convex}.
For $q=2$, the solution simplifies to a $3\times3$ convolution; for $q=1$, a local linear program provides the exact sparse solution.
Both cases achieve constant-time complexity.

\textbf{Event Sampling.} 
Since EDF is modeled as  piecewise-constant function, event intervals follow an exponential distribution and can be sampled sequentially. To handle discontinuities, a modified thinning algorithm of \cite{ogata1981lewis} incorporating event polarity is employed, also achieving $O(1)$ efficiency. Additional implementation details are provided in the supplementary material.

\subsection{The Rotary Event Dataset Benchmark}
Evaluating individual event quality is challenging because event counts and locations vary across sensors, and conventional cameras cannot provide microsecond-level reference intensity. To address this, we build a controlled high-speed rotary system that produces microsecond-accurate ground-truth intensity frames illustrated in \cref{fig:systemred}(a). The setup includes a precision motor, a binary-patterned rotating disk, uniform illumination, an event camera, and a microcontroller for sub-microsecond synchronization. After spatial alignment using printed crosses and temporal alignment via shared-clock trigger pulses, we obtain tightly synchronized events and intensity frames.



Ten binary patterns as shown in \cref{fig:systemred}(b) are printed on the disk to generate diverse motion types with uniform contrast. Spiral patterns create radial motion, while others produce rotational or mixed motions. Known pattern dimensions and the encoder’s orientation readings allow reconstruction of ground-truth intensity at any timestamp.
The motor uses a giant-magnetoresistance encoder with 120,000 steps per revolution (0.003°), calibrated using laser and photodiode measurements. Motor speed is PWM-controlled from 0 to 150 rpm.
Two representative event cameras are used: (a) EVK4 \cite{finateu20205} with 1280$\times$720 spatial resolution and 1 $\mu$s temporal resolution and, (b) DAVIS346 \cite{taverni2018front} with 346$\times$260 spatial resolution and 1 $\mu$s temporal resolution.
To account for illumination effects, sequences are recorded under two lighting conditions and three rotation speeds, producing 120 highly synchronized sequences spanning about 20min. Point-based patterns are included for evaluating direct point tracking.

\section{Experiments}
In this section, we demonstrate the effectiveness of EDFilter in preserving both scene irradiance and motion information on various tasks and downstream applications.


\textbf{Methods.}
We compare EDFilter with \textbf{Raw}, \textbf{EvFlow} \cite{wang2019ev}, \textbf{Ynoise} \cite{feng2020event}, \textbf{MLPF} \cite{guo2022low}, and \textbf{EventZoom} \cite{duan2021eventzoom}. Here, \textbf{Raw} denotes the baseline that returns the unfiltered event stream. For EDFilter, we set the parameters to $t-s = 100$ms, $\gamma = 4$ms, LUT keys ranging from $156.25\mu$s to $40$ms with a $2\times$ increment, $q = 2$, and $T_s = 10$ms for event denoising; for direct point tracking, we set $p = 1$ and $T_s = \infty$ (fully asynchronous sampling). Additional details are provided in the supplementary material.

\textbf{Datasets and Metrics.}
Besides the proposed RED dataset, we also evaluate on the E-MLB dataset \cite{ding2023mlb}, which contains extensive multi-level real-world noisy event sequences and provides the non-reference Event Structural Ratio (ESR) as an event quality metric.

For event denoising, we use mean ESR for evaluation on the E-MLB dataset and proposed to use normalized mean-squared-error (NMSE) on the RED dataset:
\begin{equation}
    \text{NMSE}=\frac{\sum_{p\in\{-1,1\}}\sum_{t,x,y}(g_{t,x,y,p}-d_{t,x,y,p})^2}{\sum_{p\in\{-1,1\}}\sum_{t,x,y} g_{t,x,y,p}^2},
\end{equation}
where $g_{t,x,y,p}$ is the ground-truth radiance change and $d_{t,x,y,p}$ is the corresponding event density with polarity $p$ at $(t,x,y)$, where we use a temporal binning of 1ms for evaluation. For the direct point tracking task, we propose to use the average tracking error (ATE) for comparison:
\begin{equation}
    \text{ATE}=\sum_i\frac{t_i-t_{i-1}}{T}\|P(x_i,y_i)-(c_x(t_i),c_y(t_i))\|_2,
\end{equation}
where $(t_i,x_i,y_i)$ is the coordinate of the $i$-th event and $(c_x(t_i),c_y(t_i))$ is the coordinate of the closest point center at time $t_i$, with $T$ being the time span of the evaluation period and $P(\cdot)$ being the projection operator from the image plane to the patterns plane.

\subsection{Event Denoising}
NMSE results on the RED dataset (\cref{tab:filteringtab}) show that our method achieves the best overall performance. This indicates that the proposed filter preserves radiance-change information more accurately than competing approaches. Mean ESR results on the E-MLB dataset (\cref{tab:filteringesr}) further confirm that our method best retains structural event information across diverse noise levels, consistent with the robustness reflected in the \textbf{std} column.

\begin{table*}[tbp]
\small
\centering
\caption{NMSE results of RED. DVS/EVK columns: DAVIS346 and EVK4 sequences. Bold red/blue: 1st/2nd best values.}
\label{tab:filteringtab}
\setlength{\tabcolsep}{1.2pt}
\begin{tabular}{clcclcclcclcclcclcclcclcclcc}\toprule
 &
   &
  \multicolumn{2}{c}{\textbf{circle1}} &
   &
  \multicolumn{2}{c}{\textbf{circle2}} &
   &
  \multicolumn{2}{c}{\textbf{circle3}} &
   &
  \multicolumn{2}{c}{\textbf{kanizsa}} &
   &
  \multicolumn{2}{c}{\textbf{patterns}} &
   &
  \multicolumn{2}{c}{\textbf{spiral1}} &
   &
  \multicolumn{2}{c}{\textbf{spiral2}} &
   &
  \multicolumn{2}{c}{\textbf{text}} &
   &
  \multicolumn{2}{c}{\textbf{std}} \\ \cline{3-4} \cline{6-7} \cline{9-10} \cline{12-13} \cline{15-16} \cline{18-19} \cline{21-22} \cline{24-25} \cline{27-28} 
\multirow{-2}{*}{\textbf{Method}} &
   &
  DVS &
  EVK &
   &
  DVS &
  EVK &
   &
  DVS &
  EVK &
   &
  DVS &
  EVK &
   &
  DVS &
  EVK &
   &
  DVS &
  EVK &
   &
  DVS &
  EVK &
   &
  DVS &
  EVK &
   &
  DVS &
  EVK \\ \hline
\textbf{Raw} &
   &
  0.106 &
  0.390 &
   &
  {\color[HTML]{0070C0} \textbf{0.087}} &
  0.166 &
   &
  {\color[HTML]{0070C0} \textbf{0.060}} &
  0.118 &
   &
  {\color[HTML]{FF0000} \textbf{0.061}} &
  0.054 &
   &
  {\color[HTML]{0070C0} \textbf{0.085}} &
  0.087 &
   &
  {\color[HTML]{0070C0} \textbf{0.137}} &
  0.139 &
   &
  {\color[HTML]{0070C0} \textbf{0.094}} &
  0.139 &
   &
  0.058 &
  0.052 &
   &
  {\color[HTML]{0070C0} \textbf{0.045}} &
  0.140 \\
\textbf{EventZoom} \cite{duan2021eventzoom} &
   &
  0.223 &
  0.338 &
   &
  0.111 &
  0.109 &
   &
  0.084 &
  0.076 &
   &
  0.210 &
  0.206 &
   &
  0.129 &
  0.129 &
   &
  0.331 &
  0.108 &
   &
  0.120 &
  0.108 &
   &
  0.283 &
  0.227 &
   &
  0.168 &
  0.130 \\
\textbf{EvFlow} \cite{wang2019ev} &
   &
  0.285 &
  0.504 &
   &
  0.265 &
  0.378 &
   &
  0.242 &
  0.361 &
   &
  0.094 &
  0.212 &
   &
  0.144 &
  0.214 &
   &
  0.529 &
  0.312 &
   &
  0.346 &
  0.312 &
   &
  0.065 &
  0.142 &
   &
  0.198 &
  0.163 \\
\textbf{MLPF} \cite{guo2022low} &
   &
  {\color[HTML]{0070C0} \textbf{0.104}} &
  {\color[HTML]{0070C0} \textbf{0.226}} &
   &
  0.091 &
  {\color[HTML]{0070C0} \textbf{0.073}} &
   &
  0.066 &
  {\color[HTML]{0070C0} \textbf{0.054}} &
   &
  {\color[HTML]{0070C0} \textbf{0.062}} &
  {\color[HTML]{0070C0} \textbf{0.049}} &
   &
  0.089 &
  {\color[HTML]{0070C0} \textbf{0.072}} &
   &
  0.250 &
  {\color[HTML]{0070C0} \textbf{0.094}} &
   &
  0.134 &
  {\color[HTML]{0070C0} \textbf{0.094}} &
   &
  {\color[HTML]{FF0000} \textbf{0.057}} &
  {\color[HTML]{0070C0} \textbf{0.049}} &
   &
  0.090 &
  {\color[HTML]{0070C0} \textbf{0.057}} \\
\textbf{Ynoise} \cite{feng2020event} &
   &
  0.244 &
  0.230 &
   &
  0.186 &
  0.125 &
   &
  0.159 &
  0.072 &
   &
  0.162 &
  0.082 &
   &
  0.173 &
  0.101 &
   &
  0.516 &
  0.115 &
   &
  0.309 &
  0.115 &
   &
  0.187 &
  0.089 &
   &
  0.160 &
  0.068 \\
\textbf{Ours} &
   &
  {\color[HTML]{FF0000} \textbf{0.096}} &
  {\color[HTML]{FF0000} \textbf{0.193}} &
   &
  {\color[HTML]{FF0000} \textbf{0.075}} &
  {\color[HTML]{FF0000} \textbf{0.059}} &
   &
  {\color[HTML]{FF0000} \textbf{0.055}} &
  {\color[HTML]{FF0000} \textbf{0.044}} &
   &
  0.071 &
  {\color[HTML]{FF0000} \textbf{0.043}} &
   &
  {\color[HTML]{FF0000} \textbf{0.079}} &
  {\color[HTML]{FF0000} \textbf{0.062}} &
   &
  {\color[HTML]{FF0000} \textbf{0.121}} &
  {\color[HTML]{FF0000} \textbf{0.086}} &
   &
  {\color[HTML]{FF0000} \textbf{0.088}} &
  {\color[HTML]{FF0000} \textbf{0.086}} &
   &
  {\color[HTML]{0070C0} \textbf{0.058}} &
  {\color[HTML]{FF0000} \textbf{0.038}} &
   &
  {\color[HTML]{FF0000} \textbf{0.036}} &
  {\color[HTML]{FF0000} \textbf{0.051}} \\\bottomrule
\end{tabular}
\end{table*}

\begin{table}[]
\small
\centering
\caption{\textbf{Mean ESR results on E-MLB}. ND\%: noise level. Bold red/blue: 1st/2nd best values.}
\label{tab:filteringesr}
\setlength{\tabcolsep}{3.5pt}
\begin{tabular}{clcccc}\toprule
\multirow{2}{*}{\textbf{Method}} & & \multicolumn{4}{c}{\textbf{E-MLB (Day/Night)}}         \\ \cline{3-6} 
                                 & & ND1       & ND4       & ND16      & ND64      \\ \hline
\textbf{Raw}                     & & 0.82/0.89 & 0.82/0.82 & 0.82/0.79 & 0.79/0.77 \\
\textbf{EventZoom} \cite{duan2021eventzoom}               & & \bblue{1.00}/\bblue{1.06} & \bblue{0.99}/\bblue{1.01} & \bblue{1.00}/\bred{1.01} & \bblue{0.97}/\bblue{0.99} \\
\textbf{EvFlow} \cite{wang2019ev}                 &  & 0.85/0.97 & 0.88/0.98 & 0.87/0.89 & 0.83/0.80 \\
\textbf{MLPF} \cite{guo2022low}                 & & 0.85/0.93 & 0.89/0.93 & 0.85/0.91 & 0.84/0.91 \\
\textbf{Ynoise} \cite{feng2020event}                  & & 0.87/1.01 & 0.86/0.94 & 0.86/0.88 & 0.82/0.79 \\
\textbf{Ours}                   & & \bred{1.02}/\bred{1.07} & \bred{1.02}/\bred{1.03} & \bred{1.02}/\bblue{1.00} & \bred{1.00}/\bred{0.99}\\\bottomrule
\end{tabular}
\tabvspace
\end{table}


Qualitative results in \cref{fig:filtering} support these findings. EvFlow is overly aggressive and removes many real events. MLPF and Ynoise retain structure better, with MLPF generally performing second-best due to its richer local input representation and MLP discriminator. EventZoom, though visually plausible, performs poorly in NMSE because its fixed discretization sacrifices temporal precision and its convolutional modules diffuse spatial energy, degrading radiance-change accuracy. In contrast, our method adapts its temporal kernel to scene dynamics and enforces spatial sparsity via motion-aware local modeling and a $q$-norm reconstruction criterion, yielding cleaner and more faithful events.

\begin{figure*}
    \centering
    \includegraphics[width=.9\linewidth]{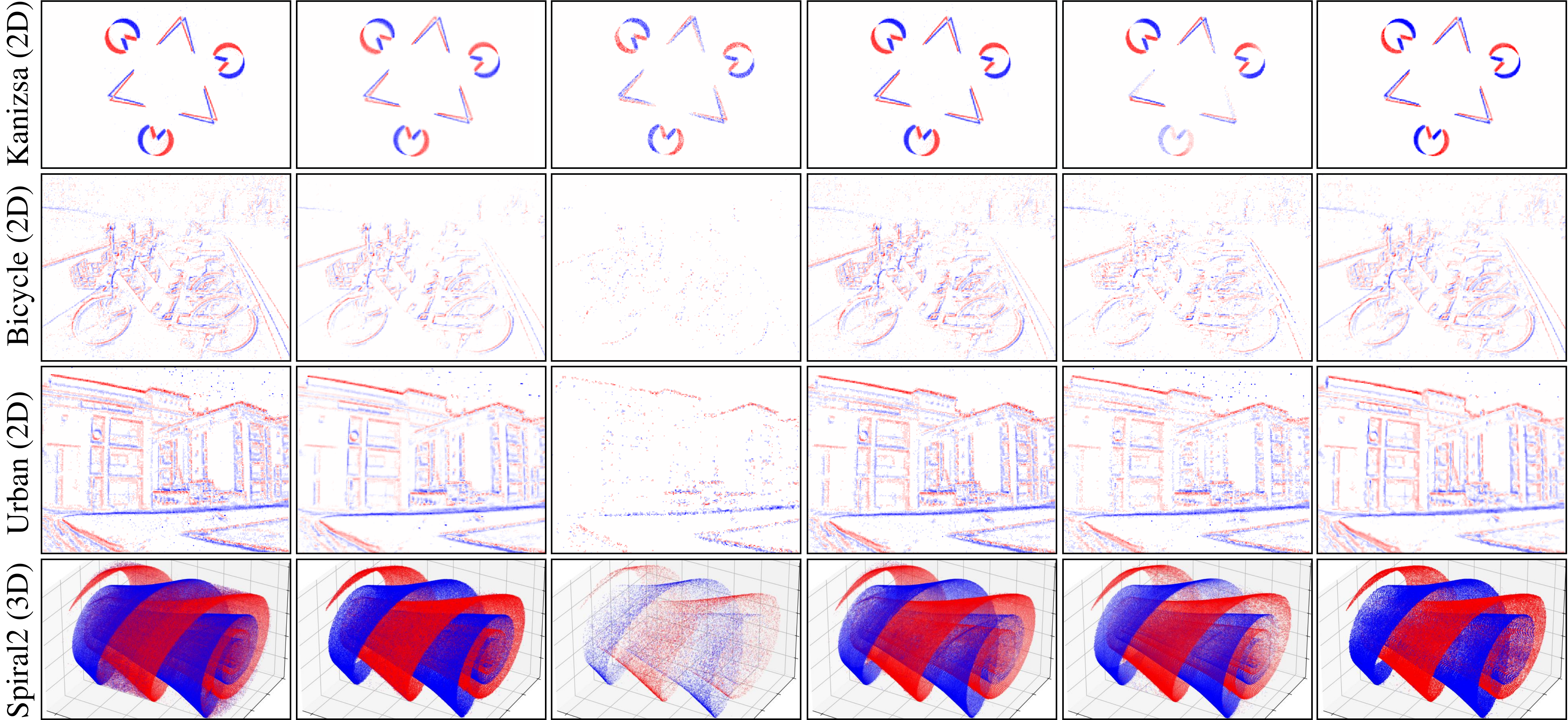}\\
    \phantom{} \hfill (a) Raw \hspace{0.5cm} (b) EventZoom \cite{duan2021eventzoom} \hspace{0.0cm} (c) EvFlow \cite{wang2019ev} \hspace{0.2cm} (d) MLPF \cite{guo2022low} \hspace{0.3cm} (e) Ynoise \cite{feng2020event} \hspace{0.7cm} (f) Ours \hspace{1.5cm}
    \caption{Visualization of filtered events as 2D event frames (top 3 rows) and 3D point clouds (bottom row). Sequences come from the proposed RED dataset (rows 1,4), the E-MLB dataset \cite{ding2023mlb} (row 2) and the ECD dataset \cite{mueggler2017event} (row 3).}
    \label{fig:filtering}
    \figvspace
\end{figure*}

\subsection{Direct Point Tracking}\label{sec:exptracking}
The direct point tracking results on ATE (\cref{tab:tracking}) show that our method ranks among the top two across all sequences. The closest competitor is EvFlow, which exhibits the lowest variance and thus the highest tracking stability. However, as evidenced by the filtering results in \cref{fig:filtering}, EvFlow performs poorly in preserving accurate irradiance changes. The 3D visualizations in \cref{fig:blob-tracking} further reveal that EvFlow retains only the most salient events along motion trajectories—beneficial for coarse motion tracking but detrimental to faithful radiance reconstruction. This highlights an intrinsic tension between accurate irradiance perception and motion perception, reflecting the separation of state and process information in conventional models.

By contrast, our method performs strongly on both fronts. This is achieved by modeling the complete event distribution through probability flux and optimizing it within a unified framework, thereby bridging the gap between state-information modeling and process-information modeling.

\begin{table}[tbp]
\small
\centering
\caption{\textbf{ATE results in millimeters (mm).} DVS/EVK columns: DAVIS346 and EVK4 sequences. Red/blue: 1st/2nd best values.}
\label{tab:tracking}
\setlength{\tabcolsep}{1.1pt}
\begin{tabular}{clcclcclcc}\toprule
                                  &  & \multicolumn{2}{c}{\textbf{point}} &  & \multicolumn{2}{c}{\textbf{points}} &  & \multicolumn{2}{c}{\textbf{std}} \\ \cline{3-4} \cline{6-7} \cline{9-10} 
\multirow{-2}{*}{\textbf{Method}} &  & DVS           & EVK              &  & DVS            & EVK              &  & DVS          & EVK             \\ \hline
\textbf{Raw}                      &  & 7.542           & 46.319           &  & 4.711            & 30.728           &  & 1.787          & 24.665          \\
\textbf{EventZoom} \cite{duan2021eventzoom}                &  & 4.168           & 4.711            &  & 3.680            & 3.084            &  & 2.096          & 2.672           \\
\textbf{EvFlow} \cite{wang2019ev} &
   &
  {\color[HTML]{FF0000} \textbf{2.290}} &
  {\color[HTML]{0070C0} \textbf{2.978}} &
   &
  {\color[HTML]{FF0000} \textbf{2.306}} &
  {\color[HTML]{0070C0} \textbf{2.125}} &
   &
  {\color[HTML]{FF0000} \textbf{0.581}} &
  {\color[HTML]{FF0000} \textbf{1.071}} \\
\textbf{MLPF} \cite{guo2022low}                     &  & 7.005           & 28.811           &  & 4.499            & 17.444           &  & 1.540          & 11.635          \\
\textbf{Ynoise} \cite{feng2020event}                  &  & 2.805           & 7.223            &  & 3.428            & 4.444            &  & 1.031          & 7.864           \\
\textbf{Ours} &
   &
  {\color[HTML]{0070C0} \textbf{2.358}} &
  {\color[HTML]{FF0000} \textbf{2.954}} &
   &
  {\color[HTML]{0070C0} \textbf{2.518}} &
  {\color[HTML]{FF0000} \textbf{2.046}} &
   &
  {\color[HTML]{0070C0} \textbf{0.810}} &
  {\color[HTML]{0070C0} \textbf{1.297}}\\\bottomrule
\end{tabular}
\tabvspace
\end{table}

\begin{figure*}
    \centering
    \includegraphics[width=.85\linewidth]{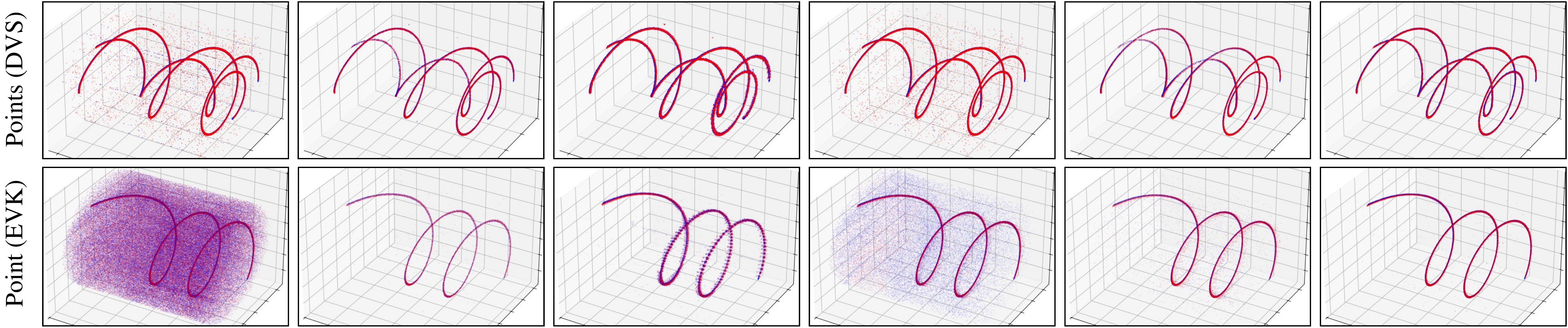}\\
    \phantom{} \hfill (a) Raw \hspace{0.3cm} (b) EventZoom \cite{duan2021eventzoom} \hspace{0.0cm} (c) EvFlow \cite{wang2019ev} \hspace{0.0cm} (d) MLPF \cite{guo2022low} \hspace{0.2cm} (e) Ynoise \cite{feng2020event} \hspace{0.5cm} (f) Ours \hspace{1.9cm}
    \caption{Visualization of filtered events in 3D as point clouds. Sequences are captured with DAVIS346 (top row) and EVK4 (bottom row).}
    \label{fig:blob-tracking}
    \figvspace
    \vspace{-1em}
\end{figure*}
\subsection{Analysis}
\hspace*{\parindent}\textbf{Ablation Studies.} Our pipeline consists of three modules. To assess their impact, we vary one component while keeping the others fixed. The averaged results for both filtering and direct point tracking are shown in \cref{fig:ablation}. Substituting the temporal model shows that the proposed EDF formulation yields the best performance on both tasks. This advantage is particularly clear in motion perception, where Poisson \cite{gu2021spatio} and Gaussian \cite{li2021asynchronous} temporal models rely on fixed correlation windows, whereas EDF adapts to scene dynamics through likelihood maximization. For spatial modeling, L1 better preserves motion accuracy, while L2 more faithfully maintains irradiance changes. Regarding the sampling strategy, asynchronous sampling gives the lowest ATE but higher NMSE than hybrid sampling (Async vs. Hybrid10ms). As discussed in \cref{sec:exptracking}, a more aggressive sampling is preferable when motion is the primary objective. Synchronous sampling performs worst across all metrics, underscoring its inefficiency for inherently asynchronous event signals. These observations also indicate that extremely fine sampling (Sync1ms vs. Sync10ms) offers limited gains relative to its computational cost.

\begin{figure}
    \centering
    \includegraphics[width=.9\linewidth]{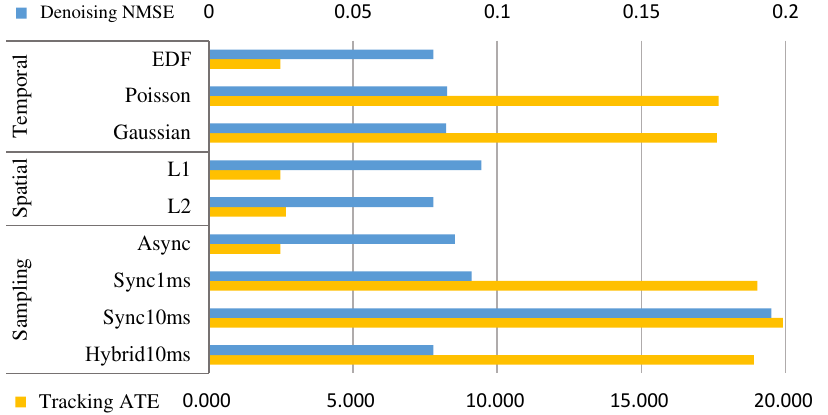}
    \caption{Ablation studies of EDFilter by its temporal modeling, spatial modeling, and sampling paradigm.}
    \label{fig:ablation}
    \figvspace
\end{figure}

\textbf{Runtime.} We evaluate the runtime characteristics of our method for online deployment. Since event streams are asynchronous, two complementary metrics are used: latency, measuring the time to produce an output after a single event arrives (online test), and throughput, measuring the sustained event-processing rate over time (offline/online test). We benchmark three representative methods---MLPF, EventZoom, and our EDFilter---under identical conditions with $346\times260$ resolution on a single core of an R9-7945HX CPU, as shown in \cref{tab:runtime}. For synchronous EventZoom, throughput is computed on real events with an average rate of 1.23 events/$\mu$s.
Our method achieves the lowest latency at 4.99 $\mu$s, outperforming MLPF by roughly 51\%, demonstrating that physically grounded signal processing can still operate at microsecond scale. EventZoom has much higher latency due to the event-frame accumulation and heavy frame-processing blocks. In terms of throughput, EDFilter can already process 210k events/s  surpassing EventZoom.


\begin{table}[tbp]
\centering
\small
\caption{\textbf{Latency and throughput comparison of event-signal filters.} Latency is measured in seconds (s), and throughput is measured in events per second (events/s).}
\label{tab:runtime}
\begin{tabular}{cccc}\toprule
\textbf{Metric}           & \textbf{MLPF \cite{guo2022low}}  & \textbf{EventZoom \cite{duan2021eventzoom}} & \textbf{Ours}  \\ \hline
\textbf{Latency}           & 7.55e-6 & 1.21e0      & {4.99e-6} \\
\textbf{Throughput} & {4.73e6} & 1.34e5      & 2.10e5 \\\bottomrule
\end{tabular}
\tabvspace
\vspace{-1em}
\end{table}

\subsection{Applications}

\hspace{\parindent}\textbf{Event-Based SLAM.}
The reference SLAM algorithm is selected as ESVO \cite{zhou2021event}, which uses stereo semi-global matching to recover camera trajectories. The performance is evaluated by comparing the computed trajectories and ground-truth trajectories on the MVSEC \cite{zhao2014mvsec} and RPG \cite{zhou2018semi} datasets. The absolute translational error (RMS) results and the top view of computed trajectories are provided in \cref{fig:appSLAM}. After inserting the proposed EDFilter, the tracking performance is improved for most sequences, especially for those previously having poor results like \textit{indoor\_flying1} and \textit{indoor\_flying3}. Such results are a direct demonstration of the motion preservation ability of the proposed method.


\begin{figure}
    \centering
    \begin{subfigure}{\linewidth}
        \centering
        \includegraphics[width=\linewidth]{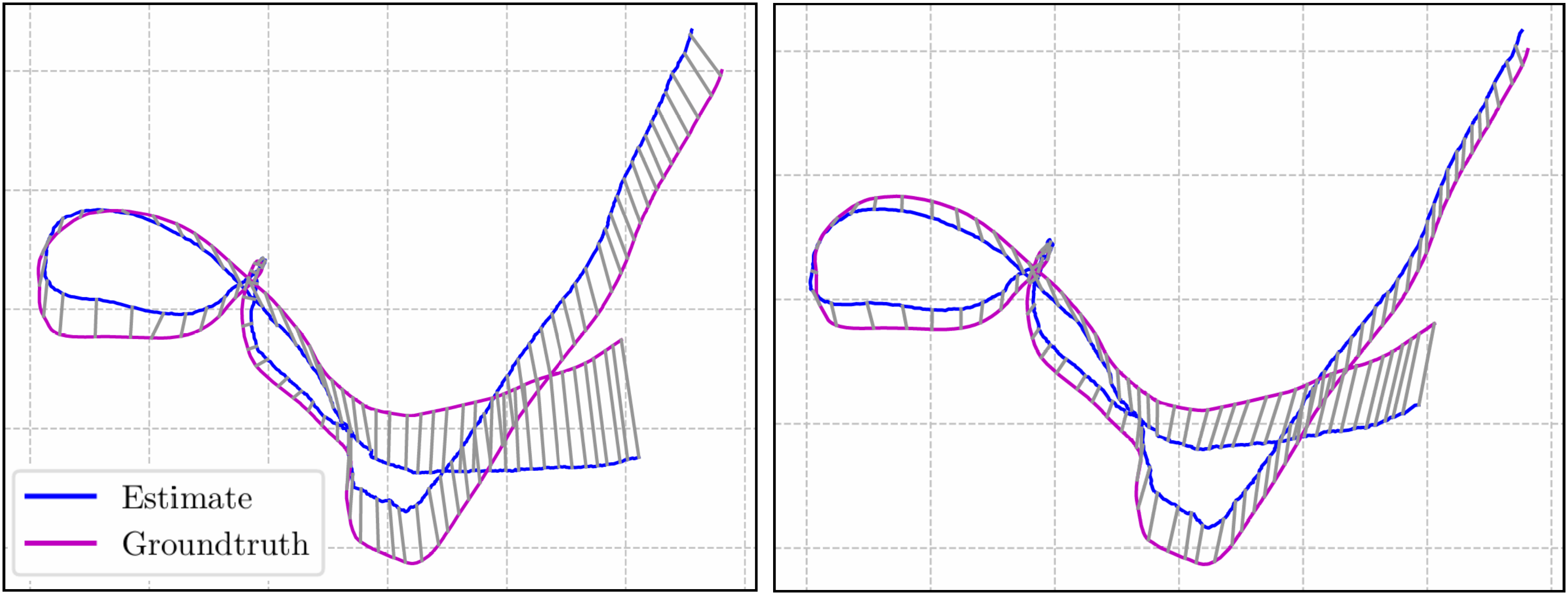}\\\vspace{-3.1cm}
        \phantom{}\hfill w/o EDFilter \hspace{2.3cm} w/ EDFilter \hspace{1.1cm}
        \vspace{2.7cm}
    \end{subfigure}\hfill
    \begin{subfigure}{\linewidth}
    \centering
    \small
    \setlength{\tabcolsep}{2pt}
        \begin{tabular}{ccccccc}\toprule
        \textbf{Method}        & \textbf{\begin{tabular}[c]{@{}c@{}}indoor\_\\ flying1\end{tabular}} & \textbf{\begin{tabular}[c]{@{}c@{}}indoor\_\\ flying3\end{tabular}} & \textbf{\begin{tabular}[c]{@{}c@{}}rpg\_\\ bin\end{tabular}} & \textbf{\begin{tabular}[c]{@{}c@{}}rpg\_\\ boxes\end{tabular}} & \textbf{\begin{tabular}[c]{@{}c@{}}rpg\_\\ desk\end{tabular}} & \textbf{\begin{tabular}[c]{@{}c@{}}rpg\_\\ monitor\end{tabular}} \\ \hline
        \textbf{w/o EDFilter}          & 18.4                                                                & 10.1                                                                & 3.0                                                          & 6.1                                                            & {\color[HTML]{FF0000} \textbf{3.5}}                           & 2.3                                                              \\
        \textbf{w/ EDFilter} & {\color[HTML]{FF0000} \textbf{15.5}}                                & {\color[HTML]{FF0000} \textbf{8.8}}                                 & {\color[HTML]{FF0000} \textbf{2.9}}                          & {\color[HTML]{FF0000} \textbf{4.8}}                            & 3.9                                                           & {\color[HTML]{FF0000} \textbf{2.2}}                             \\\bottomrule
        \end{tabular}\vspace{-.7em}
    \end{subfigure}
    \caption{SLAM results on the MVSEC \cite{zhao2014mvsec} and RPG \cite{zhou2018semi} datasets. Top: top-view trajectory of \textit{indoor\_flying1}. Bottom: RMS tracking error in centimeters (cm). Bold red: best value.}
    \figvspace
    \label{fig:appSLAM}
\end{figure}

\textbf{Video Reconstruction.}
The reference video reconstruction algorithm is selected as E2VID \cite{rebecq2019high}, which learns to generate corresponding video frames from input event frames with a recurrent network. As with evaluating other data preprocessing pipelines, we fine-tune the pretrained model on filtered events to obtain the final reconstruction results. The performance is assessed on the E-MLB \cite{ding2023mlb} dataset, reporting the average peak-to-signal-ration (PSNR), mean-squared error (MSE), structural similarity (SSIM) index, and perceptual similarity (LPIPS) \cite{zhang2018unreasonable} and reconstructed frames in \cref{fig:appRECON}. The reconstruction results have been improved for all the metrics, which is also evident visually. Such results clearly demonstrate the benefit of EDFilter in irradiance preservation.

\begin{figure}
    \centering
    \begin{subfigure}{\linewidth}
        \centering
        \includegraphics[width=\linewidth]{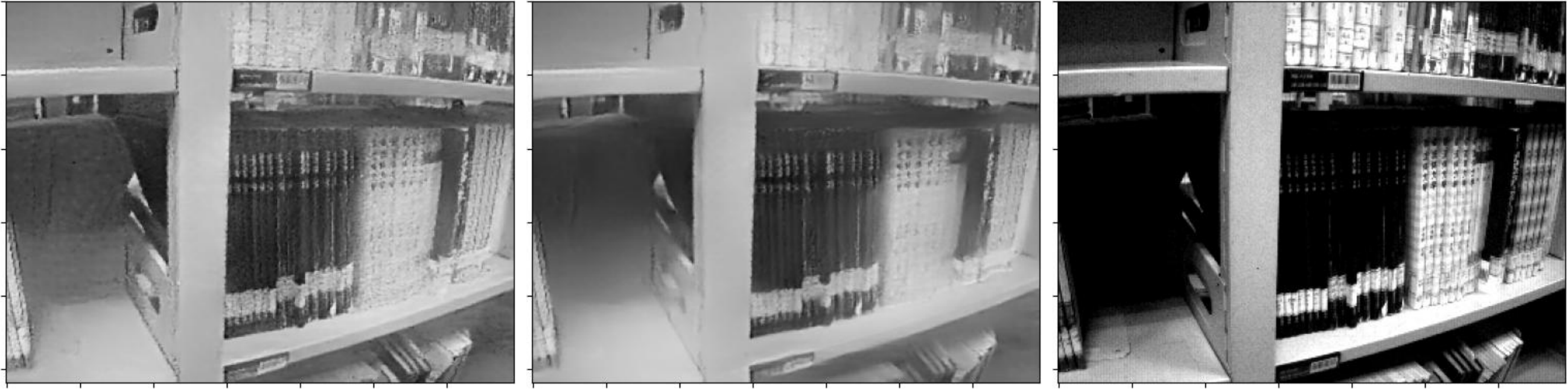}\\\vspace{-2.1cm}
        \phantom{} \hfill \white{w/o EDFilter} \hspace{.9cm} \white{w/ EDFilter} \hspace{.8cm} \white{Ground Truth} \hspace{.3cm}\vspace{1.7cm}
    \end{subfigure}
    \begin{subfigure}{\linewidth}
    \centering
    \small
        \begin{tabular}{ccccc}\toprule
        \textbf{Method} & \textbf{PSNR$\uparrow$}                         & \textbf{MSE$\downarrow$}                          & \textbf{SSIM$\uparrow$}                        & \textbf{LPIPS$\downarrow$}                       \\ \hline
        \textbf{w/o EDFilter}          & 16.04                                 & 0.035                                 & 0.53                                 & 0.30                                 \\
        \textbf{w/ EDFilter} & {\color[HTML]{FF0000} \textbf{18.37}} & {\color[HTML]{FF0000} \textbf{0.019}} & {\color[HTML]{FF0000} \textbf{0.68}} & {\color[HTML]{FF0000} \textbf{0.24}} \\\bottomrule
        \end{tabular}
        \vspace{-.7em}
    \end{subfigure}
    \caption{Video reconstruction results on the E-MLB \cite{ding2023mlb} dataset. Top: reconstructed frame. Bottom: metrics. Bold red: best value.}
    \figvspace
    \label{fig:appRECON}
    \vspace{-1em}
\end{figure}



\section{Conclusion}



We proposed EDFilter, an event signal filtering framework grounded in a probability-flux interpretation of event generation. By estimating the threshold-crossing probability flux through efficient temporal KDE, motion-aware spatial smoothing, and asynchronous resampling, EDFilter provides physically consistent and real-time event reconstruction. The RED dataset with microsecond irradiance ground truth is also captured for evaluation. Experiments show that EDFilter achieves superior denoising accuracy and tracking performance, highlighting probability-flux modeling as a promising direction for future research on generative models, physics-informed neural representations, and probabilistic formulations in event-based vision.


{
    \small
    \bibliographystyle{ieeenat_fullname}
    \bibliography{main}
}

\clearpage
\setcounter{page}{1}
\maketitlesupplementary

\section{Proof}
This section provides detailed proofs for \cref{sec:theoflux}. The main purpose is to make the text self-contained and organized for event specialists. All the symbols used here do not necessarily have the same meaning as in the main text and should be understood based on the context.

\subsection{Stochastic Differential Equation and Fokker-Planck Equation}\label{sec:sdefokkerplanck}
In \cref{sec:theoflux} we state that \cref{eq:bvpfokkerplanck} can be derived from \cref{eq:irradiancesde}. In fact, this is a direct result of the Markovian property since an SDE is already continuous, and we provide a proof in this part.

Since the process is Markovian, we have the Chapman-Kolmogorov equation:
\begin{equation}
    p(I,t)=\int p(I,t|I',t')p(I',t')dx',
\end{equation}
where $p(I,t)$ is the probability density of $I$ at time $t$ and $p(I,t|I',t')$ is the state transition function from $I'$ at $t'$ to $I$ at $t$. Since $p(I,t|I',t')$ is a probability distribution, the characteristic function of the offset $I-I'$ is:
\begin{equation}\label{eq:kramersmoyal}
\begin{split}
    \hat{p}(k,t|I',t')=\mathbb{E}[e^{ik(I-I')}]&=\int e^{ik(I-I')}p(I,t,|I',t')dI\\
    &=\sum_{n=0}^\infty \frac{(ik)^n}{n!}\mu_n(t|I',t'),
\end{split}
\end{equation}
where
\begin{equation}
    \mu_n(t|I',t')=\int (I-I')^np(I,t|I',t')dx,
\end{equation}
is the $n$-th moment. With the inverse transform, we have the formal expression:
\begin{equation}
    p(I,t|I',t')=\sum_{n=0}^\infty\frac{1}{n!}\mu_n(t|I',t')\delta^{(n)}(I-I'),
\end{equation}
where $\delta^{(n)}(x)$ is the $n$-th formal (weak) derivative of $\delta(x)$ and have the property:
\begin{equation}
    \int\delta^{(n)}(x)\phi(x)dx=(-1)^n\phi^{(n)}(0),
\end{equation}
for any $n$-th differentiable function $\phi(x)$. Setting $t'=t-\tau$ with $\tau\to0_+$, because $W_t$ is a standard Wiener process, we have:
\begin{equation}
    \lim_{\tau\to0_+}p(I,t|I',t-\tau)=\frac{e^{-\frac{(I-I'-\mu\tau)^2}{2\sigma^2\tau}}}{\sqrt{2\pi\tau\sigma^2}}.
\end{equation}
So that $I-I'$ obeys a normal distribution with mean $\mu\tau$ and variance $\sigma^2\tau$. We have:
\begin{gather}
    \mu_0(t|I',t-\tau)=1,\\
    \mu_1(t|I',t-\tau)=\mu\tau,\\
    \mu_2(t|I',t-\tau)=\sigma^2\tau+(\mu\tau)^2=\sigma^2\tau+o(\tau),\\
    \mu_{n>2}(t|I',t-\tau)=o(\tau).
\end{gather}
Substituting the results back to \cref{eq:kramersmoyal} gives:
\begin{equation}
\begin{split}
    p(I,t)&=p(I,t-\tau)+o(\tau)\\
    &+\tau\int\mu \delta^{(1)}(I-I')p(I',t-\tau)dI'\\
    &+\frac{\tau}{2}\int\sigma^2\delta^{(2)}(I-I')p(I',t-\tau)dI'.
\end{split}
\end{equation}
Moving $p(I,t-\tau)$ to the left hand and dividing both sides with $\tau$, we have:
\begin{equation}
    \partial_tp(I,t)=-\partial_I(\mu p(I,t))+\frac{1}{2}\partial_{I}^2(\sigma^2p(I,t)),
\end{equation}
which is exactly the Fokker-Planck equation in \cref{eq:bvpfokkerplanck}.

\subsection{Multidimensional Irradiance Diffusion Process}
To extend the above results to the multidimensional case with complex interaction, we start with a multidimensional stochastic differential equation:
\begin{equation}
    d\bm{I}_t=\bm{\mu}(\bm{I}_t,t)dt+\bm{\sigma}(\bm{I}_t,t)d\bm{W}_t,
\end{equation}
where $\bm{I}_t=(I(\bm{x},t)))_{\bm{x}\in[0,W)\times[0,H)}$ is the diffusion process of logarithmic scene irradiance with a spatial resolution of $W\times H$, $\bm{\mu}(\bm{I}_t,t)$ is the clean irradiance time derivative, $\bm{\sigma}(\bm{I}_t,t)$ is the thermal noise level and $\bm{W}_t$ is a standard multivariate wiener process.

Following the same derivations in \cref{sec:sdefokkerplanck}, the multidimensionial probability density function $p(\bm{I},t)$ satisfies:
\begin{gather}\label{eq:fpeq1}
    \partial_tp(\bm{I},t)=Lp(\bm{I},t)\qquad \forall\bm{I}\in D,\\\label{eq:fpeq2}
    L = -\partial_i\mu_i(\bm{I},t)+\partial_i\partial_j B_{ij}(\bm{I},t),\\\label{eq:fpeq3}
    \bm{B}(\bm{I},t)=\frac{1}{2}\bm{\sigma}(\bm{I},t)\bm{\sigma}(\bm{I},t)^T,\\\label{eq:fpeq4}
    p(\bm{I},0)=p_0(\bm{I}),\\\label{eq:boundabsorption}
    p(\bm{I},t)=0,\qquad \forall\bm{I}\in\partial D,
\end{gather}
where the Einstein summation convention is understood and $D=(-C,C)^{W\times H}$ is the contrast boundary. In this case, we can derive the corresponding probability flux density as:
\begin{gather}\label{eq:transporteq}
    \partial_tp(\bm{I},t)+\nabla\cdot\bm{J}(\bm{I},t)=0\qquad\forall\bm{I}\in D,\\\label{eq:transportinit}
    J_i(\bm{I},t)=[\mu_i(\bm{I},t)-\partial_jB_{ij}(\bm{I},t)]p(\bm{I},t),\\\label{eq:transportbound}
    p_\pm^k(t)=\int_{\{\bm{I}\in\partial D\mid I_k=C_{ON/OFF}\}}\bm{J}(\bm{I},t)\cdot\bm{n}dS,
\end{gather}
where $\bm{J}(\bm{I},t)=\{J_i(\bm{I},t)\}$ is the probability flux density of this diffusion process, $\bm{n}$ is the unit outer normal and $p_\pm^k(t)$ is the net flux of the ON/OFF event at the $k$-th pixel. 

To show that the internal and boundary behavior jointly describe all possible irradiance trajectories at any time, let $\bm{p}_\pm(t)=\{p_+^k(t),p_-^k(t)\}$, then the population of irradiance trajectories that don't generate an event is:
\begin{equation}\label{eq:remainingpopulation}
    \begin{split}
        \mathcal{N}(t):&=\int_Dp(\bm{I},t)dV\\
        &=1-\int_D\int_0^t\nabla\cdot\bm{J}(\bm{I},\tau)d\tau dV\\
        &=1-\int_0^t\int_{\partial D}\bm{J}(\bm{I},\tau)\cdot\bm{n}dSd\tau\\
        &=1-\int_0^t\|\bm{p}_\pm(\tau)\|_1d\tau,
    \end{split}
\end{equation}
where the third equation can be derived by applying Gauss's theorem using \cref{eq:transporteq}. The distribution of the state and process information of events can be derived accordingly.

\subsection{Physical Meaning of Event Density Flow}
We adopt another formulation of the net flux $p_\pm(t)$ as the event density flow $\lambda_\pm(t)$ defined in \cref{eq:eventdensityflow} without further explanation. This part discusses its physical meaning and properties to help understand the formulations in \cref{sec:algorithm}.

First, we show that $p_\pm(t)$ can be derived from $\lambda_\pm(t)$. Observing that:
\begin{gather}
    \frac{\lambda_+(t)}{\lambda_-(t)}=\frac{p_+(t)}{p_-(t)},\\
    [\lambda_-(t)+\lambda_+(t)]dt = - d\ln(1-\int_0^t[p_-(\tau)+p_+(\tau)]d\tau).
\end{gather}
It's easy to prove:
\begin{equation}\label{eq:inverserelation}
    p_\pm(t)=\lambda_\pm(t)\exp(-\int_0^t[\lambda_-(\tau)+\lambda_+(\tau)]d\tau).
\end{equation}
Since the map from probability flux to event density flow is bijective, they convey the same information. A multidimensional formulation is:
\begin{equation}
    \bm{\lambda}_\pm(t)=\frac{\bm{p}_\pm(t)}{1-\int_0^t\|\bm{p}_\pm(\tau)\|d\tau},
\end{equation}
where the inverse relation can also be derived accordingly.

Then we prove its physical meaning as the expected event rate. We take an infinitesimal interval $dt$ and treat the 1D case. Observe that:
\begin{equation}
\begin{split}
    \lambda_+(t)dt&=\frac{p_+(t)dt}{\mathcal{N}(t)}\\
    &=\frac{\mathbb{P}(t_s\in[t,t+dt])}{\mathbb{P}(t_s\geq t)}\\
    &=\mathbb{P}(t_s\in[t,t+dt]\mid t_s\geq t)\\
    &=\mathbb{E}[N_+([t,t+dt])\mid t_s\geq t],
\end{split}
\end{equation}
where the last equation holds as long as no two events can be generated at the same time, so there is either zero or one event in an infinitesimal interval. The case for the negative event and vector form is the same, which shows the physical meaning of event density flow as the expected number of events.

The asymptotic characteristics are derived from the Fokker-Planck equation. According to the operator theory, the adjoint operator of $L$ is:
\begin{equation}
    L^*=\mu_i(\bm{I},t)\partial_i+B_{ij}(\bm{I},t)\partial_i\partial_j.
\end{equation}
Then we restate the theorem in \cite{elgin1984fokker} in our case as:
\begin{proposition}\label{prop:asymptoticfokkerplanck}
    If D is a bounded domain, $\partial D$ has a piecewise continuous normal, and $L^*$ is a uniformly elliptic operator with sufficiently smooth coefficients in D, then the "steady-state" rate is the principal eigenvalue of the Fokker-Planck operator with absorbing boundary conditions.
\end{proposition}
\begin{proof}
    The Fokker-Planck equation can be solved by separation of variables,
    \begin{equation}
        p(\bm{I},t)=\sum_{n=1}^\infty\phi_n(\bm{I})e^{-\lambda_nt},
    \end{equation}
    where $\lambda_n, \phi_n(\bm{I})$ are the eigenvalues and eigenfunctions of the Fokker-Planck operator, respectively,
    \begin{align}
        L\phi_n(\bm{I})&=-\lambda_n\phi_n(\bm{I}), &\text{for}& ~ \bm{I}\in D\\
        \phi_n(\bm{I})&=0. &\text{for}& ~ \bm{I}\in\partial D
    \end{align}
    It is known that $\lambda_1>0$ and $\phi_1(\bm{I})>0$ in $D$ and $\lambda_1$ is a simple eigenvalue. It follows that
    \begin{equation}
        \mathcal{N}(\bm{I},t)=\int_Dp(\bm{I},t)dV=\sum_{n=1}^\infty\int_D\phi_n(\bm{I})dVe^{-\lambda_nt},
    \end{equation}
    hence
    \begin{equation}
        \dot{\mathcal{N}}(\bm{I},t)=-\sum_{n=1}^\infty\lambda_n\int_D\phi_n(\bm{I})dVe^{-\lambda_nt},
    \end{equation}
    and
    \begin{equation}
        -\frac{\dot{\mathcal{N}}(\bm{I},t)}{\mathcal{N}(\bm{I},t)}=\frac{\sum_{n=1}^\infty\lambda_n\int_D\phi_n(\bm{I})dVe^{-\lambda_nt}}{\sum_{n=1}^\infty\int_D\phi_n(\bm{I})dVe^{-\lambda_nt}}.
    \end{equation}
    It follows that
    \begin{equation*}
        \lim_{t\to\infty}\|\bm{\lambda}_\pm(t)\|_1=\lim_{t\to\infty}-\frac{\dot{\mathcal{N}}(\bm{I},t)}{\mathcal{N}(\bm{I},t)}=\lambda_1.
    \end{equation*}
\end{proof}

\cref{prop:asymptoticfokkerplanck} provides the mathematical basis for the formulation of the event density flow, which is also why we also use $\beta_\pm$ to constitute the temporal EDF model in \cref{eq:temporaledf}.

\section{Algorithm}
This section provides some details of the optimization results in \cref{sec:algorithm} to reduce the complexity of our algorithm.

\subsection{Temporal EDF Model}
This section provides proofs required to perform optimization in \cref{sec:temporal}. First the expression of the event observation likelihood is proved.

\begin{proposition}
    Given a point pattern $(t_1,\ldots,t_n)$ on an observation interval $[0,T)$, the likelihood function is given by
    \begin{equation}
        \mathcal{L} = \left(\prod_{i=1}^n\lambda(t_i)\right)\exp\left(-\int_0^T\lambda(\tau)d\tau\right).
    \end{equation}
    Given a marked point pattern $((t_1,\kappa_1),\ldots,(t_n,\kappa_n))$ on $[0,T)\times\mathbb{M}$, the likelihood function is given by
    \begin{equation}\label{eq:supplikelihood}
        \mathcal{L} = \left(\prod_{i=1}^n\lambda(t_i,\kappa_i)\right)\exp\left(-\int_0^T\|\bm{\lambda}(\tau)\|_1d\tau\right).
    \end{equation}
\end{proposition}

\begin{proof}
    The likelihood function is the joint density function of all the points $(t_1,\ldots,t_n)$ in the observation period $[0,T)$.
    \begin{equation}
        \mathcal{L}=\prod_{i=1}^np(t_{i-1},t_i)\left(1-\int_{t_n}^Tp(t_n,\tau)d\tau\right),
    \end{equation}
    where $p(t_s,t_e)$ is the probability density of an event at $t_e$ evolving from $t_s$ and we set $t_0=0$ as the start of observation. From \cref{eq:inverserelation}, we have:
    \begin{equation}
        \begin{split}
            \mathcal{L}&=\prod_{i=1}^n\lambda(t_i)\exp\left(-\int_{t_{i-1}}^{t_i}\lambda(\tau)d\tau\right)\\
            &\quad\times\left(1-\int_{t_n}^T\lambda(t)\exp\left(-\int_{t_n}^t\lambda(\tau)d\tau\right)dt\right)\\
            &=\left(\prod_{i=1}^n\lambda(t_i)\right)\exp\left(-\int_0^{t_n}\lambda(\tau)d\tau\right)\\
            &\quad\times\left(1+\left.\exp\left(-\int_{t_n}^t\lambda(\tau)d\tau\right)\right|_{t_n}^T\right)\\
            &=\prod_{i=1}^np(t_{i-1},t_i)\left(1-\int_{t_n}^Tp(t_n,\tau)d\tau\right).
        \end{split}
    \end{equation}
    The marked cased can be established accordingly by observing that the last part of the likelihood should be replace by $\mathcal{N}(t_n,T)$ to denote the probability that no event occurs in $(t_n,T)$.
\end{proof}

Based on the above result, it's possible to prove the following propositions.
\begin{proposition}\label{prop:kernelnorm1}
    For an observed event pattern within $[0,W)$ of the same polarity $\{t_i,p\}$ and an nonnegative casual kernel as $\phi(t)$, if the density flow of correspond polarity is defined as:
    \begin{equation}
        \lambda_{pol}(t)=\sum_{t_i<t}c\phi(t-t_i),
    \end{equation}
    where $c>0$ is an arbitrary scale. Then the MLE of $c_{\mathrm{MLE}}\to1/\int_0^\infty\phi(s)ds$ as $W\to\infty$.
\end{proposition}
\begin{proof}
    Let the number of observed events till $W$ is $n(W)$, then the likelihood is:
    \begin{equation}
        \begin{split}
            \mathcal{L}&=\left(\prod_{i=1}^{n(W)}\lambda(t_i)\right)\exp\left(\int_0^W\lambda(\tau)d\tau\right)\\
             &=\left(\prod_{i=1}^{n(W)}\sum_{j=1}^ic\phi(t_i-t_j)\right)\exp\left(\sum_{i=1}^{n(W)}\int_0^Wc\phi(\tau-t_i)d\tau\right)\\
             &=\Phi c^{n(W)}\exp\left(c\sum_{i=1}^{n(W)}\int_0^{W-t_i}\phi(s)ds\right).
        \end{split}
    \end{equation}
    By maximizing this likelihood with respect to $c$, we have
    \begin{equation}
        c_{\text{MLE}}=\frac{n(W)}{\sum_{i=1}^{n(W)}\int_0^{W-t_i}\phi(\tau)d\tau}\to\frac{1}{\int_0^\infty\phi(\tau)d\tau},
    \end{equation}
    as $W\to\infty$.
\end{proof}

\begin{proposition}
    Given an observed event pattern within $[0,W)$ of $\{t_i,p_i\}$ and density kernel as $\phi(t)$, if the number of false event determined by $\psi(ts)pol\leq0$ is $N_{pol}$ and the time span of $\pm\psi(t)\leq0$ is $L_\pm\leq W$, then the false event rate maximizing the observation likelihood is:
    \begin{equation}
        \beta_{\pm,\text{MLE}} = \frac{N_\pm}{L_\pm}.
    \end{equation}
\end{proposition}
\begin{proof}
    The contribution of $\beta_\pm$ to the final likelihood is:
    \begin{equation}
        \mathcal{L_{\beta_\pm}}=\beta_\pm^{N_\pm}\exp\left(-\beta_\pm L_\pm\right).
    \end{equation}
    The by maximizing this likelihood, we have
    \begin{equation}
        \beta_{\pm,\text{MLE}}=\frac{N_\pm}{L_\pm}.
    \end{equation}
\end{proof}

\begin{proposition}
    If only one event of $(ts,pol)$ is observed with $\alpha T=1$, then the MLE of the density flow with corresponding polarity is:
    \begin{equation}
        \lambda_{pol}(t)=\begin{cases}
            \frac{1}{\gamma+t-ts}, & \text{if $0\leq t-ts<(e-1)\gamma$}\\
            0. & \text{otherwise}
        \end{cases}
    \end{equation}
\end{proposition}
\begin{proof}
    Without loss of generality, let $ts=0$ and $pol=+1$, then $\psi(t)=\alpha[u(t-\frac{1}{\alpha})-u(t))]$ where $u(t)$ is the unit step function. When we observe the event at time $t$, the likelihood is
    \begin{equation}
        \mathcal{L}=\gamma\exp(-\alpha\gamma)\times\begin{cases}
            \alpha\exp(-\alpha t), & \text{if $\frac{1}{\alpha}> t$}\\
            \alpha\exp(-1). & \text{otherwise}
        \end{cases}
    \end{equation}
    When $\frac{1}{\alpha}>t$, the $\alpha$ that maximizes the likelihood is
    \begin{equation}
        \alpha_1=\frac{1}{t+\gamma}<\frac{1}{t}.
    \end{equation}
    Then the likelihood is
    \begin{equation}
        \mathcal{L}_1=\frac{\gamma}{t+\gamma}\exp(-1).
    \end{equation}
    When $\frac{1}{\alpha}\leq t$, the $\alpha$ that maximizes the likelihood is
    \begin{equation}
        \alpha_2=\begin{cases}
            \frac{1}{\gamma}, &\text{if $\gamma<t$}\\
            \frac{1}{t}. &\text{otherwise}
        \end{cases}
    \end{equation}
    Then the likelihood is:
    \begin{equation}
        \mathcal{L}_2=\begin{cases}
            \exp(-2), &\text{if $\gamma<t$}\\
            \frac{\gamma}{t}\exp(-\frac{\gamma}{t})\exp(-1). &\text{otherwise}
        \end{cases}
    \end{equation}
    It's easy to prove that $\frac{\gamma}{t+\gamma}\exp(-1)>\frac{\gamma}{t}\exp(-\frac{\gamma}{t})\exp(-1)$, so the only chance that $\frac{1}{\alpha}\leq t$ is
    \begin{equation}
        \exp(-2)\geq\frac{\gamma}{t+\gamma}\exp(-1)\Rightarrow t\geq(e-1)\gamma.
    \end{equation}
    Since $\lambda(t)=0$ when $t\geq\frac{1}{\alpha}$, $\lambda_{\text{MLE}}(t)=0$ when $t\geq(e-1)\gamma$. Otherwise, $\lambda_{\text{MLE}}=\alpha_1=\frac{1}{t+\gamma}$.
\end{proof}

\subsection{The Spatial EDF Model}
This section provides proofs for the uniqueness of the output mapping and the sparsity-preserving property for the spatial model in \cref{sec:spatial}. To start with, we define the problem in a general case.

\begin{definition}
    A convex optimization problem is defined as finding the minimum of a convex function where the feasible set is also a convex set.
\end{definition}

\begin{proposition}\label{prop:suppcvxsol}
    A convex optimization problem has zero, one, or many solutions, which must form a convex set called optimal set.
\end{proposition}

\begin{proof}
    Let $C$ be the feasible set and $f(\mathbf{x})$ be the target convex function. Then the convex optimization problem is the problem of finding some $\mathbf{x}^*\in C$ attaining $\inf\{f(\mathbf{x}):\mathbf{x}\in C\}$. The the solution space can be defined as the sublevel set $\{\mathbf{x}:f(\mathbf{x})\leq\inf\{f(\mathbf{x}),\mathbf{x}\in C\},\mathbf{x}\in C\}$. Since the sublevel sets of convex functions are convex, the solution space must be a convex set.
\end{proof}

\begin{definition}\label{def:suppl1problem}
    A finite dimensional linear L1 norm minimization problem is defined as
    \begin{align}
        \underset{\mathbf{k}}{\operatorname{minimize}} \qquad & \lVert\mathbf{a}-\mathbf{k}\mathbf{B}\rVert_1\\
        \operatorname{subject~to}\qquad &\mathbf{a}\in\mathbb{R}^n\\
        &\mathbf{k}\in\mathbb{R}^m\\
        &\mathbf{B}\in\mathbb{R}^{m\times n}.
    \end{align}

    The optimal set is denoted as $\mathcal{O}_\mathbf{a}$.
\end{definition}

\begin{definition}
    The characteristic linear function of the optimization problem in \cref{def:suppl1problem} is defined as 
    \begin{equation}
        \{f:f(\mathbf{k})=\mathbf{f}\cdot\mathbf{k},|f(\mathcal{O}_\mathbf{a})|=1,\forall \mathbf{a}\in\mathbb{R}^n\},
    \end{equation}
    where $\mathbf{f}\in\mathbb{R}^m$ is the representation of $f$.
\end{definition}

Now that we have the required definitions, some important results relating to the characteristic linear function space are provided below.
\begin{proposition}
    The characteristic linear function space of the optimization problem in \cref{def:suppl1problem} is a linear space.
\end{proposition}

\begin{proof}
    Let $\forall\mathbf{k}\in\mathcal{O}_\mathbf{a}$ and $f,g$ be the characteristic linear functions, then $\forall \alpha,\beta\in\mathbb{R}$, $(\alpha f+\beta g)(\mathbf{k})=\alpha f(\mathbf{k})+\beta g(\mathbf{k})$ is also unique, so $\alpha f+\beta g$ is also a characteristic linear function.
\end{proof}

\begin{proposition}\label{prop:suppuniqueness}
    If the matrix $\mathbf{B}=(b_{ij})_{1\leq i\leq m,1\leq j\leq n}$ contains at most one nonzero element other than one column, $\sum_{j=1}^n(-1)^{n_j}b_{ij}\neq 0$ for $\forall i\in\{1,2,\ldots,m\}$ and $\forall n_j\in\{0,1\}$, let it be the $j$-th column, then the linear function $f(\mathbf{k})=\sum_{i=1}^mb_{ij}k_i$ is a characteristic linear function of the optimization problem of \cref{def:suppl1problem}.
\end{proposition}

\begin{proof}
     Use induction. When $m=1$, the target function $g_\mathbf{a}(k)=\sum_{i=1}^n|a_i-kb_i|$ is a convex piecewise linear function, and the slope of each piece $\sum_{i=1}^n(-1)^{n_i}b_i\neq 0$, so $g_\mathbf{a}$ is a strictly convex function and the solution is unique. Linear functions $kb_i,\forall i\in\{1,2,\ldots,m\}$ are the characteristic linear functions.

     Assuming the result holds for $m<M$, let $b_{M,j}\neq 0$. When $k_i,i<M$ are fixed, $f(\mathbf{k})$ is a strictly convex piecewise linear function of variable $k_M$. The best $k_M$ could be any turning point as $k_i,i<M$ varies. So $k_M\in\{a_l/b_{i,l}:b_{i,l}\neq 0,l\neq j\}=S_M$ or $k_M=(a_j-\sum_{i=1}^{M-1}k_ib_{ij})/b_{M,j}$.

     If $k_M$ in the optimal set $\mathcal{O}_\mathbf{a}$ is unique, when $k_M\in S_M$, we have by induction that $\sum_{i=1}^{M-1}k_ib_{ij}$ is unique, so $\sum_{i=1}^Mk_ib_{ij}$ is unique. Otherwise $k_M=(a_j-\sum_{i=1}^{M-1}k_ib_{ij})/b_{M,j}$, then $\sum_{i=1}^Mk_ib_{ij}=a_j$. If $k_M=k_{M1},k_M=k_{M2},k_{M1}\neq k_{M2}$ are in $\mathcal{O}_\mathbf{a}$, select $\mathbf{k}_1=(k_{11},k_{21},\ldots,k_{M1})\in\mathcal{O}_\mathbf{a}$ and $\mathbf{k}_2=(k_{12},k_{22},\ldots,k_{M2}) \in \mathcal{O}_\mathbf{a}$, then $\mathbf{k}_\alpha=\alpha\mathbf{k}_1+(1-\alpha)\mathbf{k}_2\in\mathcal{O}_\mathbf{a},\forall0\leq\alpha\leq1$ because $\mathcal{O}_\mathbf{a}$ is a convex set. When $\alpha$ is selected such that $\alpha k_{M1}+(1-\alpha)k_{M2}\notin S_M$, $k_M=(a_j-\sum_{i=1}^{M-1}k_ib_{ij})/b_{M,j}$ or $\sum_{i=1}^Mk_ib_{ij}=a_j$. The equality holds almost everywhere for $\alpha\in[0,1]$, or $\sum_{i=1}^m(\alpha k_{i1}+(1-\alpha)k_{i2})b_{ij}\overset{a.e.}{=}a_j,\forall\alpha\in[0,1]$. If $\exists\alpha\in[0,1]$ such that $\alpha k_{M1}+(1-\alpha)k_{M2}\in S_M$, because $\sum_{i=1}^Mk_ib_{ij}$ is unique when $k_M\in S_M$, so $\sum_{i=1}^M(\alpha k_{i1}+(1-\alpha)k_{i2})=a_j$ holds everywhere. So the linear function $f(\mathbf{k})=\sum_{i=1}^Mk_ib_{ij}$ is the characteristic linear function of the optimization problem \cref{def:suppl1problem}.
\end{proof}

\begin{lemma}
    For a local density flow map around an empty pixel:
    \begin{equation}
        A=\begin{bmatrix}
        a_{11} & a_{12} & a_{13}\\
        a_{21} & 0 & a_{23}\\
        a_{31} & a_{22} & a_{33}
        \end{bmatrix},
    \end{equation}
    if the surrounding densities are also sparse:
    \begin{equation}
        0=a_{11}a_{33}=a_{12}a_{32}=a_{13}a_{31}=a_{21}a_{23},
    \end{equation}
    then the filtering output defined by \cref{eq:spatialbasis,eq:spatialopt,eq:spatialcontrib,eq:spatialout} is $D=0$, \ie, no density is diffused to the center pixel.
\end{lemma}
\begin{proof}
    Without loss of generality, let $0=a_{11}=a_{12}=a_{13}=a_{21}$. Notice that
    \begin{equation}
    \begin{split}
        \|A-\sum_{i=1}^4k_iB_i\|_1=|a_{33}-k_1|+|k_1|+|a_{22}-k_2|\\+|k_2|
        +|a_{31}-k_3|+|k_3|+|a_{23}-k_4|\\+|k_4|+|\sum_{i=1}^4b_ik_i|
        \\\leq|a_{33}|+|a_{22}|+|a_{31}|+|a_{23}|,
    \end{split}
    \end{equation}
    where the equality is obtained when $k_i=0$. Since $D$ is unique from \cref{prop:suppuniqueness}, $D=\sum_{i=1}^4b_i\times0=0$.
\end{proof}

\subsection{Event Sampling}
From \cref{eq:processlikelihood,eq:statelikelihood}, the state and process distribution of event signal need to be independently computed from the event density flow using \cref{eq:inverserelation}, which is not easy considering the complexity of the obtained event density flow. However, it turns out that direct sampling methods exist to directly sample output events from the density flow with polarities treated independently. The result is the Ogata's modified thinning algorithm with polarities as marks in \cref{alg:sampling}.
\begin{algorithm}
    \caption{Ogata's modified thinning algorithm for sampling events from event density flow}\label{alg:sampling}
    \begin{algorithmic}[1]
    \Require Event denstiy flow $(\lambda_+(t),\lambda_-(t))$, sampling period $[t_s,t_e)$.
    \State $t_+=t_-=t_s$, $n_+=n_-=0$
    \While{$t_+<t_e$}
        \State Select $m_+(t)\geq\sup_{s\in[t_s,t_e)}\lambda_+(s)$ and step $l_+(t)>0$
        \State Generate independent random variables $s\sim\text{Exp}(m_+(t))$ and $U\sim\text{Unif}([0,1])$
        \State If $s>l_+(t)$, set $t_+=t_++l_+(t)$
        \State Else if $t_++s>t_e$ or $U>\lambda_+(t_++s)/m(t)$, set $t_+=t_++s$
        \State Otherwise, set $n_+=n_++1$, $ts_{n_+}=t+s$, $pol_{n_+}=+1$, $t_+=t_++s$
    \EndWhile
    \State Repeat the above process for $t_-,n_-$ to sample output events with $pol=-1$
    \end{algorithmic}
\end{algorithm}

Provided all the above issues are resolved, ordinary interpolation methods can be applied. For ease of selecting $m_\pm(t),l_\pm(t)$, zero-order hold is applied to interpolate the continuous event density flow, so that $m_\pm(t)$ can be selected as the height and $l_\pm(t)$ can be selected as the time span of each constant density flow piece.

\subsection{The Sequential EDFilter Algorithm}\label{sec:asyncimplementation}
To asynchronously compute the likelihood, we first prove that it's state space representation.

According to \cref{sec:temporal}, the likelihood \cref{eq:likelihood} is characterized by $\alpha$, $h$ and $\beta_\pm$ for each pixel independently. Because there is no simple solution when the number of observed events is large, look-up tables (LUTs) are instead used to find the optimal parameters. The idea is to create LUTs where likelihood values are associated with parameters as keys and then select the parameters with the maximum likelihood. Fortunately, for a fixed $\alpha$ and $\alpha h=1$ according to \cref{prop:kernelnorm1}, the likelihood can be efficiently computed with asynchronous state update equations.


For a fixed $\alpha$, denote the incoming events at this pixel $\{\ldots,(t_i,p_i),\ldots\}$ as $E(t)=\sum_i p_i\delta(t-t_i)$, then the predicted event density flow $\psi(t)=E(t)*\phi(t)$. Since this is a linear time-invariant (LTI) system, there is a state-space representation for asynchronously updating the predicted event density flow, where the minimal realization is:
\begin{align}
    \frac{d\psi(t)}{dt}=\alpha(E(t)-E(t-\frac{1}{\alpha})),
\end{align}
and the corresponding state update equations are:
\begin{gather}
    \psi(t_2)=\psi_\beta(t_1)+\alpha p_{t_2},\label{eq:async1}\\
    \psi(t)=\psi_\beta(t_1), \forall t\in[t_1,t_2),\label{eq:async2}
\end{gather}
where $t_1$ and $t_2$ are the timestamps of adjacent events in the composite pulse train $E(t)-E(t-\frac{1}{\alpha})$ and $p_{t_2}$ is the polarity of the composite event at $t_2$. Given these asynchronous state update equations, the number of events that doesn't comply with the predicted density $N_\pm(t)$ and the time span $L_\pm(t)$ can be asynchronously updated as:
\begin{gather}
\begin{split}
    N_{p_{t_2}}(t_2)&=N_{p_{t_2}}(t_1)\\
    &+\begin{cases}
        1, & \text{if $t_2$ real and $\psi(t_2)p_{t_2}\leq0$}\\
        0, & \text{otherwise}
    \end{cases},\label{eq:async3}
\end{split}\\
    N_{pol}(t)=N_{pol}(t_1),\qquad\forall t\in[t_1,t_2).\label{eq:async4.1}\\
    L_\pm(t)=L_\pm(t_1)+\begin{cases}
        t-t_1, &\text{if $\pm\psi(t_1)>0$}\\
        0. &\text{otherwise}
    \end{cases},\label{eq:async4.2}
\end{gather}

Then the likelihood can also be asynchronously computed as:
\begin{gather}\label{eq:async5}
\begin{split}
        &\mathcal{L}_s(t_2)=\mathcal{L}_s(t_1)\\
        &\times\begin{cases}
        |\psi(t_2)|, & \text{if $t_2$ real and $\psi(t_2)p_{t_2}>0$}\\
        1, & \text{otherwise}
    \end{cases}
\end{split}\\
    \label{eq:async6}
    \mathcal{L}_s(t)=\mathcal{L}_s(t_1),\qquad\forall t\in[t_1,t_2)\\\label{eq:async7}
    \mathcal{L}_p(t)=\mathcal{L}_p(t_1)e^{-(t-t_1)|\psi(t_1)|},\quad\forall t\in[t_1,t_2)\\
\label{eq:async8}
    N_\pm(t_s,t_e)=N_\pm(t_e)-N_\pm(t_s),\\\label{eq:async9}
    \beta_\pm(t_s,t_e)=\frac{N_\pm(t_e)-N_\pm(t_s)}{L_\pm(t_e)-L_\pm(t_s)},\\\label{eq:async10}
    \mathcal{L}_{process}(t_s,t_e)=\frac{\mathcal{L}_p(t_e)}{\mathcal{L}_p(t_s)}e^{-N_+(t_s,t_e)-N_-(t_s,t_e)},\\\label{eq:async11}
\begin{split}
    &\mathcal{L}_{state}(t_s,t_e)=\frac{\mathcal{L}_s(t_e)}{\mathcal{L}_s(t_s)}\\
    &\times[\beta_+(t_s,t_e)]^{N_+(t_s,t_e)}[\beta_-(t_s,t_e)]^{N_-(t_s,t_e)}
\end{split}
\end{gather}

Since we choose constant observation $W=t_e-t_s$, $\mathcal{L}_p(t_s)$ can be computed using a delayed event spike train $E(t-W)$ as input to the above state update equations, thereby achieving the asynchronous computation of $\mathcal{L}(t_s,t_e)$ with $O(1)$ complexity. As a result LUTs can be asynchronously updated and used for fast solution search. The simplest form is just selecting the LUT index with largest event observation likelihood.

The detailed process of the proposed $O$(1) Implementation of EDFilter in provided in \Cref{alg:asyncedfilter}.

\begin{algorithm}
    \caption{Sequential EDFilter Algorithm}\label{alg:asyncedfilter}
    \begin{algorithmic}[1]
    \Require $W,\gamma$, LUT keys $\{\alpha_i\}_{i=1}^N$, synchronous sampling period $dt$, synchronous sampling start $t_0$.
    \State $t_0\gets0$
    \State $\lambda_{\pm,i}(t_0,x,y)\gets0,i=1\ldots N$
    \While{Current time $t<t_0+dt$}
    \While{New event coming}
        \State $(ts,x,y,p)\gets$ new event
        \For{$\{(u,v)|\|(u,v)-(x,y)\|_\infty\leq 2\}$}
            \State Asynchronously update LUT density flow $\lambda_{\pm,i}(ts,u,v)$ according to \cref{eq:eventdensityflow,eq:async1,eq:async2}
            \State Asynchronously update LUT likelihood $\mathcal{L}_{u,v,i}(ts-W,ts)$ according to \cref{eq:likelihood,eq:async3,eq:async4.1,eq:async4.2,eq:async5,eq:async6,eq:async7,eq:async8,eq:async9,eq:async10,eq:async11}
            \State Find the best index $j=\arg\max_i \mathcal{L}_{u,v,i}(ts-W,ts)$, and obtain the best density flow estimation $\bar{\lambda}_\pm(ts,u,v)=\lambda_{\pm,j}(ts,u,v)$
        \EndFor
        \For{$\{(u,v)|\|(u,v)-(x,y)\|_\infty\leq 1\}$}
            \State Construct local density flow map around $(u,v)$ as $A_\pm(ts,u,v)=(\bar{\lambda}_\pm(ts,u+k,v+l))_{-1\leq k,l\leq1}$ 
            \State Get filtered density flow $\lambda^*_\pm(ts,u,v)=\sum_{i=1}^4k_i(ts,u,v)c_i(ts,u,v)$ according to \cref{eq:spatialbasis,eq:spatialopt,eq:spatialcontrib,eq:spatialout}
            \State Sample filtered events back using \cref{alg:sampling} with zero-order-hold to interpolate continuous event density flow $(\lambda_+^*(t,u,v),\lambda_-^*(t,u,v))$.
        \EndFor
    \EndWhile
        \State $t_0\gets t_0+dt$
        \For{All $\{(x,y)\}$ in plane}
            \State Asynchronously update LUT density flow $\lambda_{\pm,i}(t_0,x,y)$ according to \cref{eq:eventdensityflow,eq:async1,eq:async2}
            \State Asynchronously update LUT likelihood $\mathcal{L}_{x,y,i}(t_0-W,t_0)$ according to \cref{eq:likelihood,eq:async3,eq:async4.1,eq:async4.2,eq:async5,eq:async6,eq:async7,eq:async8,eq:async9,eq:async10,eq:async11}
            \State Find the best index $j=\arg\max_i \mathcal{L}_{x,y,i}(t_0-W,t_0)$, and obtain the best density flow estimation $\bar{\lambda}_\pm(ts,x,y)=\lambda_{\pm,j}(ts,x,y)$
            \State Get filtered density flow $\lambda^*_\pm(t_0,x,y)=\sum_{i=1}^4k_i(t_0,x,y)c_i(t_0,x,y)$ according to \cref{eq:spatialbasis,eq:spatialopt,eq:spatialcontrib,eq:spatialout}.
            \State Sample filtered events back using \cref{alg:sampling} with zero-order-hold to interpolate continuous event density flow $(\lambda_+^*(t,x,y),\lambda_-^*(t,x,y))$.
        \EndFor
    \EndWhile
    
    \end{algorithmic}
\end{algorithm}

\section{Experiment}
This section provides additional experimental details.

\subsection{The Motion System Setup}
The whole system is built around an accurate motion system to generate ground truth scene radiance change for fair comparison of different event signal processing methods. To simplify the physical model, we use only one motor to drive a flat disk with printed binary patterns attached.

For precise motion sensing, the motor is equipped with a two-channel 500 Pulses Per Revolution (PPR) giant magnetoresistance motor encoder and a 1/60 gearbox. The two channels have a phase shift of 90\textdegree, resulting in quadcounts enhanced resolution. After all, the expected count number per revolution is $4\times500\times60=120000$, resulting in a 0.003\textdegree angular step. However, reduction ratio is not accurately 1/60, and gear backlash exists in the motor and gearbox, so an external refinement is required.

We avoid explicitly modeling the gear backlash by maintaining the same direction of rotation during the capture since gear backlash mainly exists when the direction changes. To calibrate the gear reduction ratio, we use an aiming laser, a photoresistor and a sticky note for external calibration as shown in \cref{fig:suppcalib}. A sticky note periodically blocks the laser beam shooting at the photoresistor, which is captured by the microcontroller. The motor encoder is also connected to the microcontroller. By counting the average number of counts per rotation indicated by the photoresistor, we are able to calibrate the reduction ratio. As a result, the calibrated counts per rotation is $122727.4\pm0.1$ for our motor setup.

\begin{figure}
    \centering
    \includegraphics[width=\linewidth]{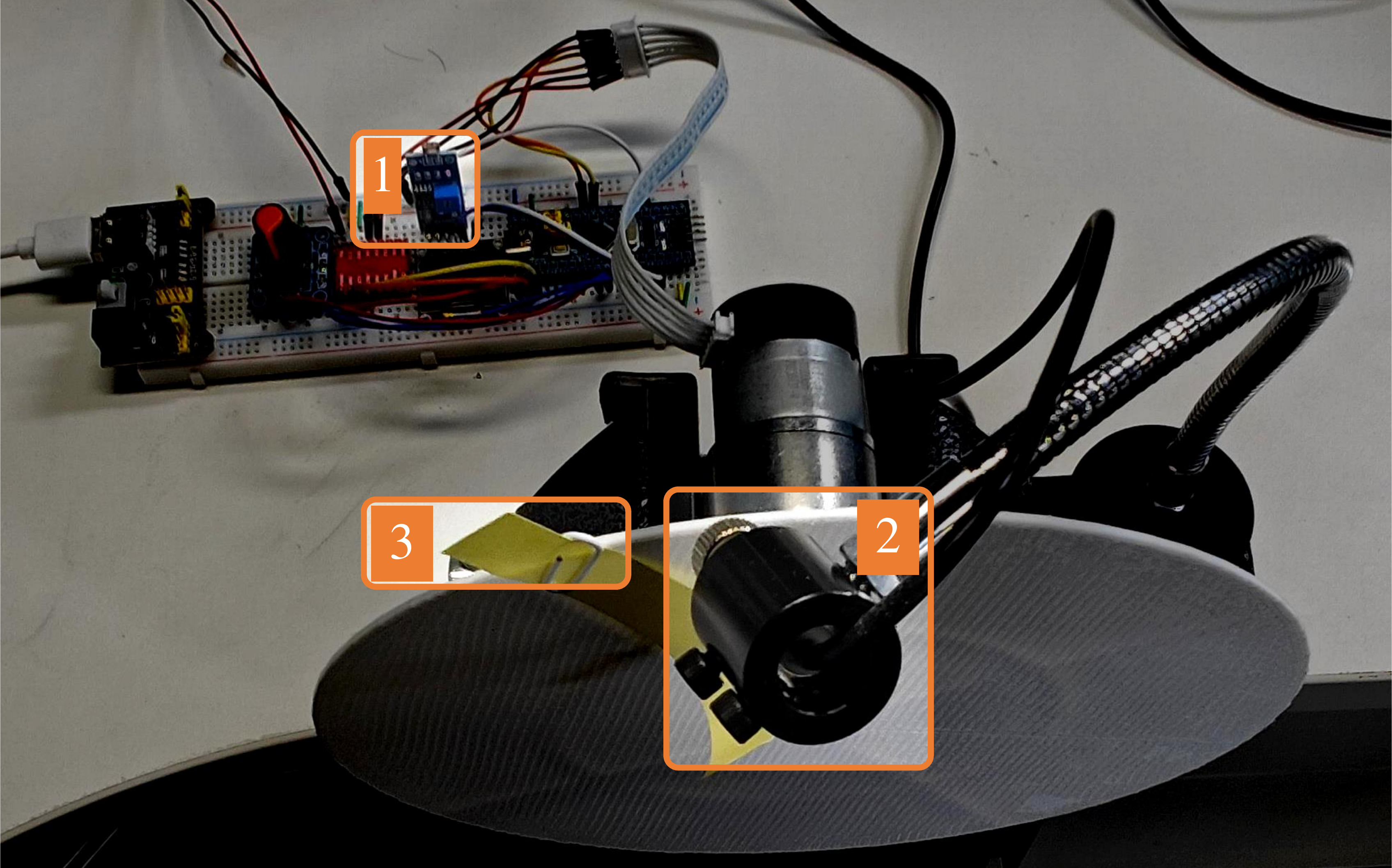}
    \caption{\textbf{Gear calibration setup.} A laser (2) is aiming at a photoresistor (1). As the disk rotates, a sticky note (3) periodically blocks the laser beam.}
    \label{fig:suppcalib}
\end{figure}

\subsection{Temporal Synchronization Setup}
Due to the microsecond-level temporal resolution of event cameras, precise temporal synchronization is not easy by only analyzing the captured events. 
Instead, we directly interact with the event timing hardware to achieve temporal synchronization.

The DAVIS346\cite{taverni2018front} and EVK4\cite{finateu20205} event cameras provide a way to synchronize external triggers with events by directly interacting the event timing hardware. Each time a rising/falling edge on the trigger port is detected, it is stamped with the internal camera clock and can be readout from the output event stream as special trigger events.

The system setup is shown in \cref{fig:suppsync}. We use the microcontroller to generate random synchronization pulses, which is captured by both the event camera through the trigger port and a 100MHz logic analyzer. The output pulses from the motor encoder is also directly attached to the logic analyzer to avoid data relaying. As a result, we have two timing regions: the logic analyzer region and event camera region. By calibrating the time shift between the two timing regions using the same synchronization signal, we are able to achieve very precise temporal synchronization. As a result, the standard deviation of temporal synchronization is $10\pm2\mu$s for EVK4 and $60\pm10\mu$s for DAVIS346. The performance difference may come from the different event timing circuit and the event timing payload. We also find that the event timing hardware requires around 100$\mu$s for a complete scan of event pixels, so we choose to evaluate the performance for every 1ms.

\begin{figure}
    \centering
    \includegraphics[width=\linewidth]{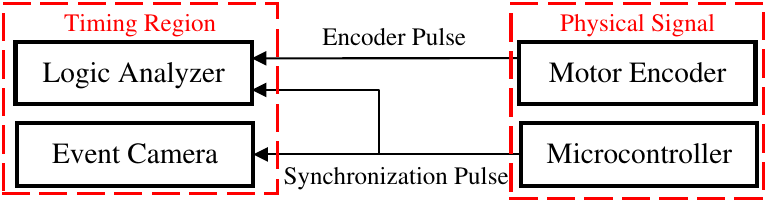}
    \caption{\textbf{Temporal synchronization setup.} Two physical signals (Encoder Pulse and Synchronization Pulse) are connected to two timing regions (Logic Analyzer and Event Camera).}
    \label{fig:suppsync}
\end{figure}

\subsection{Spatial Synchronization Setup}
In the spatial synchronization process, we aim to project all the captured signals into the same world coordinate system. The disk coordinate system is chosen for easier incorporating of physical measurements.

The problem exists as how to project the captured events onto the disk. We do this by manual calibration with calibration patterns as shown in \cref{fig:suppmarks}. Four cross marks are printed along with the texture for spatial calibration. Because the printed texture center is no the rotation center, an eccentric circle model is used for spatial synchronization. By manually labelling the cross marks on the projected event frame, we are able to decide the rotation center and shift in the eccentric circle model, thus achieving spatial synchronization. As a result, the standard deviation is around 0.1 pixel. For evaluation, all the events generated by calibration patterns are manually cropped, resulting in a clean event stream.

\begin{figure}
    \centering
    \includegraphics[width=\linewidth]{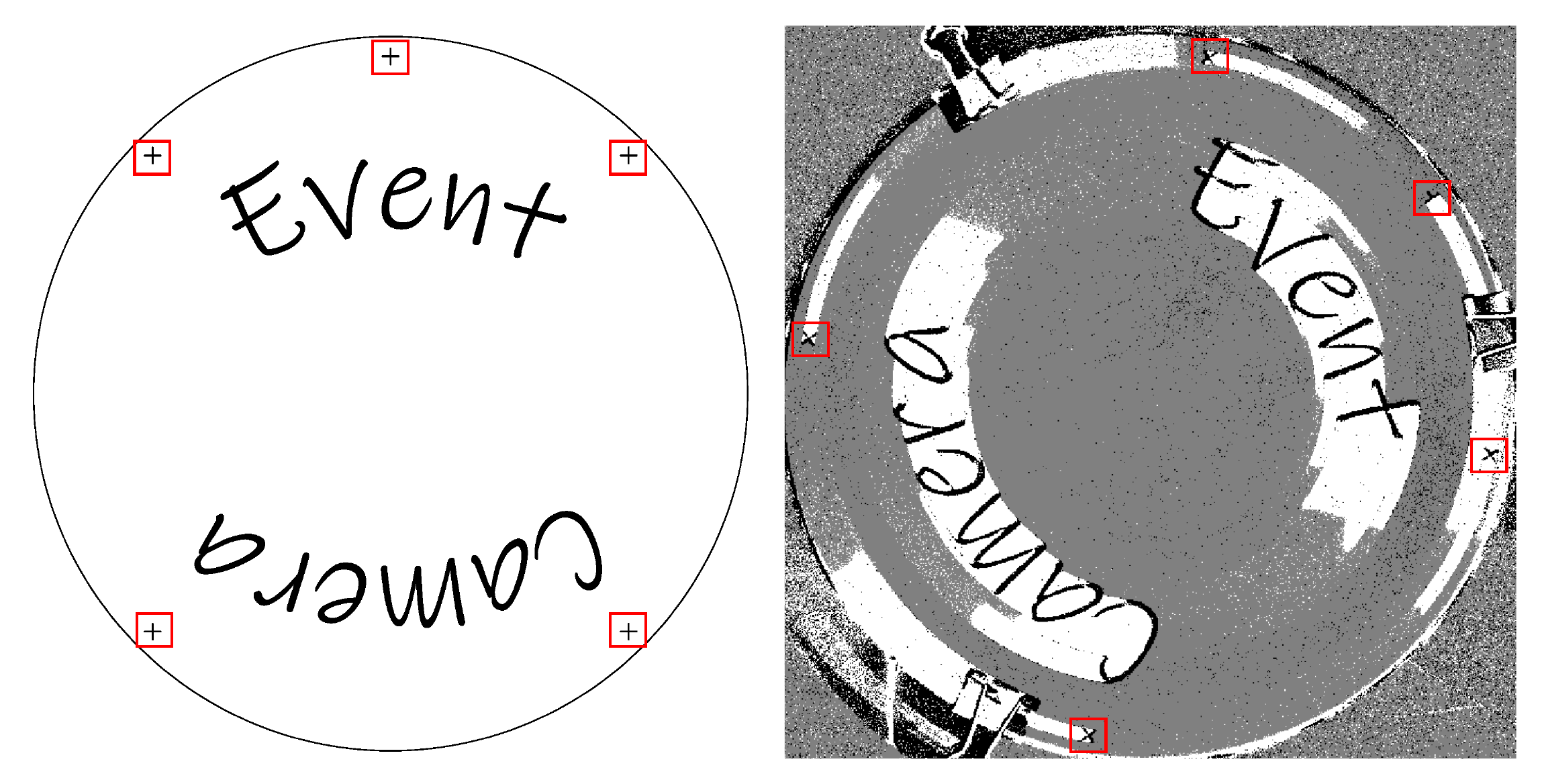}
    \caption{\textbf{Spaital synchronization setup.} Four cross marks are also printed for spatial synchronization by manual labelling on the projected event frame.}
    \label{fig:suppmarks}
\end{figure}

\subsection{Scene Radiance Change Generation}
After the temporal and spatial synchronization process, we are able to capture the precise motion in the system. The next step is to generate ground truth radiance change.

Since we use binary texture pattern under uniform illumination, the logarithm scene radiance at each pixel is either 0 or 1 (not accounting for scale). We use back projection to generate ground truth radiance change within arbitrary time range. Specifically, we are able to calculate the trajectory of each pixel in that time range. By sampling the texture intensity along the trajectory, we obtain a binary sequence with each rising edge representing the positive radiance change and falling edge representing the negative radiance change. By counting the number of rising/falling edges, we are able to obtain the ground truth scene radiance change. Using back-projection avoids computationally intensive full-image warping operations since we can select which pixel to warp and it's easier to parallelize.

\subsection{Compared Methods and Metrics}
In detail, the compared event signal filtering algorithms are:
\begin{itemize}
    \item \textbf{Raw} is the baseline method that returns raw event stream without filtering.
    \item \textbf{EvFlow \cite{wang2019ev}} is a motion-based event denoising method that passes events having reasonable optical flow, which is estimated by local plane fitting. The parameters are selected as $vmax=10p/ms$, $locality=3p\times3p\times3ms$.
    \item \textbf{Ynoise \cite{feng2020event}} is a density-based event denoising method that passes events with large event density over a spatio-temporal neighborhood. The parameters are selected as $deltaT=10ms$, $lParam=3$, $threshold=2$;
    \item \textbf{MLPF \cite{guo2022low}} is a timestamp-based event denoising method that uses a 3-layer MLP classification network to output the real probability of each event based on stacked local timesurface input. The discrimination threshold is selected as $realThres=0.5$.
    \item \textbf{EventZoom \cite{duan2021eventzoom}} is a frame-based event denoising and interpolation method that uses 3D-UNET to predict denoised and super-resolved event frames. The output events are returned from the output frame by even redistribution within the frame interval. The output threshold parameter is selected as $th=0.5$.
\end{itemize}

For evaluation on NMSE, we note that our method is generative and thus doesn't necessarily give the same result each time. To alleviate the impact of this random effect, the detailed algorithm for NMSE calculation is:
\begin{equation}
\small
    \textrm{NMSE} = \frac{\min_k \sum_{p\in\{-1,1\}}\sum_i(g_{i,p}-kd_{i,p})^2}{\sum_{p\in\{-1,1\}}\sum_ig_{i,p}^2},
\end{equation}
which contains another global scale $k$ to be optimized.

For ATE, we still sample output events and compare their tracking performance since it's already normalized. For ESR computation, we note that the metric in \cite{ding2023mlb} is based on an assumption of a fixed number of events. To alleviate the effect of random event sampling, we directly utilize the computed event rate (or EDF) to calculate ESR. The expressions are:
\begin{equation}
    NTSS=\sum_{i=1}^{W\times H}p_i^2,
\end{equation}
\begin{equation}
    L_N=W\times H-\sum_{i=1}^{W\times H}\exp(-Mp_i),
\end{equation}
\begin{equation}
    ESR=NTSS\times L_N
\end{equation}
where
\begin{equation}
    p_i=\frac{\lambda_{-,i}+\lambda_{+,i}}{\sum_{i=1}^{W\times H}(\lambda_{-,i}+\lambda_{+,i})},
\end{equation}
is the normalized EDF.

\subsection{Super Resolution}
Our method is also capable of generating super-resolved events since we are, in fact, modeling the continuous event density flow, and classical interpolation methods can be applied.

Given the $2\times2$ local density flow, we use the $2\times$ super-resolution result using bilinear kernel is:
\begin{align}\label{eq:bilinearinterop}
    \begin{bmatrix}
        a_{11} & \frac{a_{11}+a_{12}}{2} & a_{12} & a_{12}\\
        \frac{a_{11}+a_{21}}{2} & \frac{a_{11}+a_{12}+a_{21}+a_{22}}{4} & \frac{a_{12}+a_{22}}{2} & \frac{a_{12}+a_{22}}{2}\\
        a_{21} & \frac{a_{21}+a_{22}}{2} & a_{22} & a_{22}\\
        a_{21} & \frac{a_{21}+a_{22}}{2} & a_{22} & a_{22}\\
    \end{bmatrix},
\end{align}
where $(a_{ij})_{0\leq i,j\leq1}$ is the low-resolution local density flow. The above result can be regarded as the interpolation using single-sided reflected density map and can be extended to larger blocks and more complex interpolation kernels.

The results of 2$\times$ super-resolution on NMSE are shown in \cref{tab:superresolution}. The proposed method has the lowest NMSE value for most sequences, and is close to EventZoom for the remaining sequences. It's exciting to see that ours has such performance just by adapting bilinear interpolation to the asynchronous case from \cref{eq:bilinearinterop}, which shows great potential of the proposed framework. Based on the same principle, any target resolution is available by changing the interpolation kernel, making it a common adaptor design for processes that only accept certain spatial resolution or need a variable setup of computational load.

However, such design has inefficiencies when dealing with the sparse nature of events as shown in real-world sequences in \cref{fig:super-resolution}. Diffused density flow is usually expected compared with EventZoom which implicitly learns to restore scene structure from large-scale training pairs. To that end, a more detailed spatial density flow model is required for complex scenes.



\begin{table*}[]
\small
\centering
\caption{NMSE results of 2$\times$ super-resolution, lower is better. DVS/EVK columns: DAVIS346 and EVK4 captured sequences. Bold red: best value.}
\label{tab:superresolution}
\setlength{\tabcolsep}{1.2pt}
\begin{tabular}{clcclcclcclcclcclcclcclcclcc}\toprule
                                    &  & \multicolumn{2}{c}{\textbf{circle1}}                                          &  & \multicolumn{2}{c}{\textbf{circle2}}                                          &  & \multicolumn{2}{c}{\textbf{circle3}}                                          &  & \multicolumn{2}{c}{\textbf{kanizsa}}                                          &  & \multicolumn{2}{c}{\textbf{patterns}}                                         &  & \multicolumn{2}{c}{\textbf{spiral1}}                                          &  & \multicolumn{2}{c}{\textbf{spiral2}}                                          &  & \multicolumn{2}{c}{\textbf{text}}                                             &  & \multicolumn{2}{c}{\textbf{std}}                                              \\ \cline{3-4} \cline{6-7} \cline{9-10} \cline{12-13} \cline{15-16} \cline{18-19} \cline{21-22} \cline{24-25} \cline{27-28} 
\multirow{-2}{*}{\textbf{Method}} &  & DVS                                & EVK                                   &  & DVS                                & EVK                                   &  & DVS                                & EVK                                   &  & DVS                                & EVK                                   &  & DVS                                & EVK                                   &  & DVS                                & EVK                                   &  & DVS                                & EVK                                   &  & DVS                                & EVK                                   &  & DVS                                & EVK                                   \\ \hline
\textbf{EventZoom \cite{duan2021eventzoom}}             &  & 0.190                                 & 0.298                                 &  & 0.103                                 & 0.077                                 &  & 0.070                                 & {\color[HTML]{FF0000} \textbf{0.048}} &  & 0.174                                 & 0.166                                 &  & 0.121                                 & 0.102                                 &  & 0.130                                 & {\color[HTML]{FF0000} \textbf{0.061}} &  & {\color[HTML]{FF0000} \textbf{0.077}} & {\color[HTML]{FF0000} \textbf{0.075}} &  & 0.224                                 & 0.160                                 &  & 0.094                                 & 0.102                                 \\
\textbf{Ours}                  &  & {\color[HTML]{FF0000} \textbf{0.103}} & {\color[HTML]{FF0000} \textbf{0.193}} &  & {\color[HTML]{FF0000} \textbf{0.078}} & {\color[HTML]{FF0000} \textbf{0.066}} &  & {\color[HTML]{FF0000} \textbf{0.059}} & 0.056                                 &  & {\color[HTML]{FF0000} \textbf{0.076}} & {\color[HTML]{FF0000} \textbf{0.050}} &  & {\color[HTML]{FF0000} \textbf{0.087}} & {\color[HTML]{FF0000} \textbf{0.068}} &  & {\color[HTML]{FF0000} \textbf{0.119}} & 0.075                                 &  & 0.088                                 & 0.086                                 &  & {\color[HTML]{FF0000} \textbf{0.055}} & {\color[HTML]{FF0000} \textbf{0.054}} &  & {\color[HTML]{FF0000} \textbf{0.037}} & {\color[HTML]{FF0000} \textbf{0.050}}\\\bottomrule
\end{tabular}
\end{table*}

\begin{figure}
    \centering
    \includegraphics[width=\linewidth]{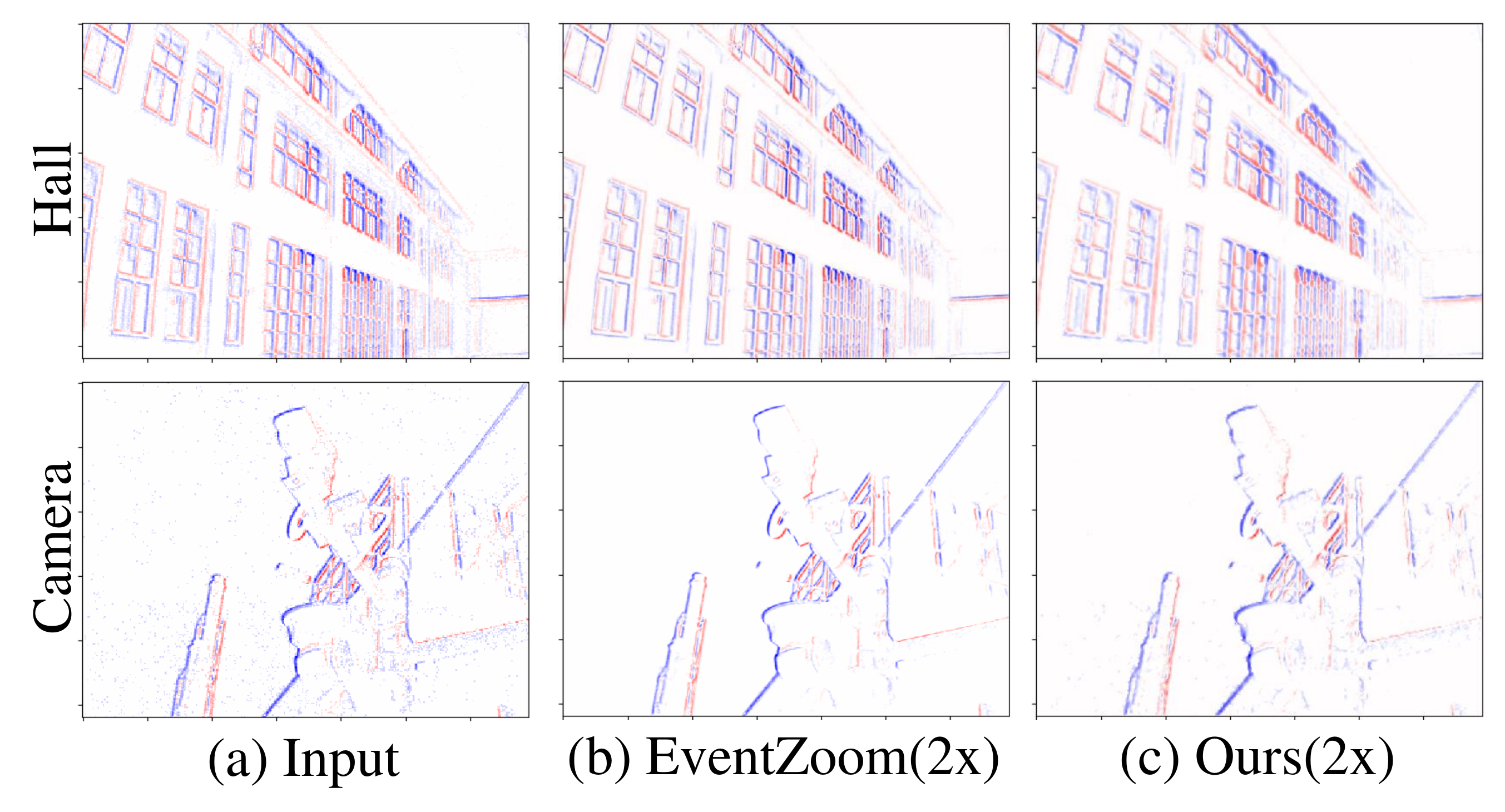}
    \caption{2$\times$ super-resolution results on the E-MLB\cite{ding2023mlb} dataset.}
    \label{fig:super-resolution}
\end{figure}

\subsection{Detailed Runtime Performance Analysis}
The proposed method is composed of three parts: the prediction part, the update part and the sampling part. Although each part is designed to be of O(1) computational complexity, the bases are not the same. Since the proposed sequential implementation in \cref{sec:asyncimplementation} enables O(1) performance by using LUT keys, the runtime performance is also dependent on it. As shown in \cref{tab:detailedperformance}, each input event will trigger 25 predictions completed in 477ns$\times$\#NKeys time and 9 updates completed in 554ns. The runtime performance of the sampling part is rather dependent on the output events, with each output event consuming about 146ns. When NKeys=9, this results in a latency of 4.99us if the number of output events is about the same as the input.

\begin{table}[]
\centering
\caption{Runtime latency performance of different parts.}
\label{tab:detailedperformance}
\begin{tabular}{ccc}
Prediction & Update & Sampling \\ \hline
477ns$\times$\#NKeys           &  554ns      &         146ns
\end{tabular}
\end{table}

\end{document}